\documentclass[a4paper,12pt]{article}
\usepackage[ruled]{algorithm2e}
\usepackage{times}
\usepackage{helvet}
\usepackage{courier}
\usepackage{enumerate}
\usepackage{amsmath}
\usepackage{url}
\usepackage{amsmath}
\usepackage{amssymb}

\newcommand{\var}{\varphi}

\newcommand{\lto}{\leftarrow}
\newcommand{\lrto}{\leftrightarrow}
\newcommand{\Rto}{\Rightarrow}
\newcommand{\Lto}{\Leftarrow}
\newcommand{\LRto}{\Leftrightarrow}
\newtheorem{definition}{Definition}
\newtheorem{examp}{Example}
\newenvironment{example}{\begin{examp}\rm}{\end{examp}}
\newcommand{\Eed}{\hfill$\Box$}
\newtheorem{lemma}{Lemma}
\newtheorem{proposition}{Proposition}
\newtheorem{theorem}{Theorem}
\newtheorem{corollary}[theorem]{Corollary}
\newenvironment{proof}{{\bf Proof:}}{\hfill\rule{2mm}{2mm}\\ }

\long\def\comment#1{}

\newcommand{\Pos}{\textit{Pos}}

\newcommand{\Neg}{\textit{Neg}}

\newcommand{\Var}{\textit{Var}}

\newcommand{\Forget}{{\sf Forget}}

\newcommand{\IR}{\textit{IR}}

\newcommand{\wunfold}{\textit{unfold}}
\newcommand{\unfold}{\textit{sunfold}}
\newcommand{\PI}{\mbox{PI}}
\newcommand{\IP}{\mbox{IP}}
\newcommand{\forget}{{\sf Forget}}
\newcommand{\Mod}{{\sf Mod}}

\newcommand{\res}{\textrm{res}}
\newcommand{\ren}{\textrm{ren}}
\newcommand{\Extension}[1]{_{\dagger#1}}
\newcommand{\LA}{{\cal L}_{\cal A}}
\newcommand{\VarEQ}{\textsc{var-equivalence}}
\newcommand{\VarIND}{\textsc{var-independence}}
\newcommand{\VarWeak}{\textsc{var-weak}}
\newcommand{\VarStrong}{\textsc{var-strong}}
\newcommand{\VarEnt}{\textsc{var-entailment}}
\newcommand{\VarMatch}{\textsc{var-match}}

\newcommand{\NP}{\textmd{\rm NP}}
\newcommand{\coNP}{\textmd{\rm co-NP}}
\newcommand{\PIP}[1]{$\Pi_{#1}^{\rm P}$}

\title{On Forgetting in Tractable Propositional Fragments}%{}%퉀\author{Yisong Wang\\
Department of Computer Science and Technology, \\
Guizhou University, China, 550025\\
}
\date{}

\begin{document}
\maketitle

\begin{abstract}
Distilling from a knowledge base only the part that is
relevant to a subset of alphabet, which is recognized as \emph{forgetting}, has attracted extensive interests
in AI community. In standard propositional logic, a general algorithm of forgetting and its computation-oriented
investigation in various fragments whose satisfiability are tractable are still lacking.
The paper aims at filling the gap.
After exploring some basic properties of forgetting in propositional logic, we present a resolution-based
algorithm of forgetting for CNF fragment, and some complexity results about forgetting in
Horn, renamable Horn, q-Horn, Krom, DNF and CNF fragments of propositional logic.
%Some application of these complexity results are also considered.
\end{abstract}


\section{Introduction}
Motivated from Lin and Reiter's  seminal work in first-order logic \cite{Fangzhen:forgetit},
the notion of \emph{forgetting} \---  distilling from a knowledge base
only the part that is relevant to a subset of the alphabet \---
has attracted extensive interests \cite{Lang:JAIR:2003,Lang:AIJ:2010}. A dual notion
of forgetting in mathematical logic is called {\em uniform interpolation} \cite{Visser:Godel:1996}.
In artificial intelligence,
it has been studied under many different names including variable eliminating,
irrelevance, independence, irredundancy, novelty, or separability \cite{AIJ:1997}.

In recent years, researchers have  developed forgetting notions and theories in
other non-classical logic systems, such as forgetting in logic programs under
answer set/stable model semantics
\cite{DBLP:Zhang:AIJ2006,DBLP:Eiter:AIJ:2008,Wong:PhD:Thesis,Yisong:KR:2012,Yisong:IJCAI:2013},
forgetting in description logic \cite{Wang:AMAI:2010,Lutz:IJCAI:2011}, and
knowledge forgetting in modal logic
\cite{Yan:AIJ:2009,Kaile:JAIR:2009,Yongmei:IJCAI:2011}.
It is commonly recognized that forgetting has both theoretical and practical interest as it can be used
for conflict solving \cite{DBLP:Zhang:AIJ2006,Lang:AIJ:2010} and knowledge compilation \cite{Yan:AIJ:2009,Bienvenu:AAAI:2010}, and
it is also closely related to other logical notions, including strongest
necessary and weakest sufficient conditions \cite{DBLP:Lin:AIJ:2001}, strongest and weakest
definitions \cite{Lang:AIJ:2008} and so on.
%The applications of forgetting

Though forgetting has been extensively investigated from various aspects of
different logical systems, in standard propositional logic, a general algorithm of forgetting and
its computation-oriented investigation in various fragments whose satisfiability
are tractable are still lacking.

Firstly, the syntactic forgetting operator, which is defined
as $\Forget(\Sigma,p)=\Sigma[p/\top]\lor\Sigma[p\bot]$ where
$\var[p/\top]$ (resp. $\var[p/\bot]$) is obtained from $\var$ be replacing $p$ with $\top$ (resp. $\bot$),
results in a disjunctive formula. Thus, it violates categoricity for non-disjunctive
formulas, e.g., if $\Sigma$ is a conjunctive
normal form (CNF) formula then $\Forget(\Sigma,p)$ is not a  CNF formula any longer. Though one can
transform a CNF formula into an equivalent disjunctive normal form (DNF) formula, the translation will bring about exponential
explosion if no fresh atoms are allowed.

Secondly, from the perspective of computation, Lang \textit{ et al}. have showed
that deciding if a formula is independent of a set of atoms (called
\VarIND) is \coNP-complete, and deciding if two formulas are equivalent on a common signature (called \VarEQ) is
\PIP{2}-complete \cite{Lang:JAIR:2003}. To our best knowledge,
such two reasoning problems remain unknown for many interesting fragments of propositional logic, such as
Horn \cite{Horn:JSL:1951}, renamable Horn \cite{Lewis:JACM:1978} (ren-Horn in short),
q-Horn theories \cite{Boros:AMAI:1990,Lang:AIJ:2008} and so forth.

In the paper we mainly focus on CNF fragments of propositional logic, for which
a resolution-based algorithm of forgetting is presented at first. Accordingly, we show that forgetting is categorical in the
Horn, ren-Horn, q-Horn, double Horn \cite{Eiter:IC:1998} and Krom \cite{Krom:JSL:1970} (or 2-CNF) fragments.
Namely, the result of forgetting from a Horn (resp. ren-Horn, q-Horn, double Horn and Krom \cite{Krom:JSL:1970} (or 2-CNF))
theory is Horn (resp. ren-Horn, q-Horn, double Horn and Krom)  expressible.

More importantly, from the perspective of knowledge bases evolving, we are also interested
in the following reasoning problems about forgetting, % in those tractable fragments,
besides the \VarIND\ and \VarEQ\ in \cite{Lang:JAIR:2003}, where $\Forget(\var,V)$ stands for
a result of forgetting $V$ from formula $\var$,
\begin{itemize}
  \item[(1)] [\VarIND] If a knowledge base $\Pi$ is independent of a set $V$ of atoms, i.e. $\Forget(\Pi,V)\equiv\Pi$.
  \item[(2)] After a knowledge base $\Sigma$ has evolved from a knowledge base $\Pi$ by incorporating
  some knowledge additionally on a set $V$ of new propositions,
  \begin{itemize}
  \item[] [\VarWeak] if the restriction of $\Sigma$
  on the signature  of $\Pi$ is at most as strong as $\Pi$, i.e. $\Pi\models\Forget(\Sigma,V)$.
  \item[] [\VarStrong]  if the restriction of $\Sigma$
  on the signature of $\Pi$ is at least as strong as $\Pi$, i.e. $\Forget(\Sigma,V)\models\Pi$.
  \item[] [\VarMatch] if the restriction of $\Sigma$
  on the signature of $\Pi$ perfectly matches $\Pi$, i.e. $\Forget(\Sigma,V)\equiv\Pi$.
  This is also known as  the forgetting result checking, i.e. if $\Pi$ is a result
  of forgetting $V$ from $\Sigma$.
  \end{itemize}

  \item[(3)] After two knowledge bases $\Pi$ and $\Sigma$ have evolved from a common knowledge base by
  incorporating  some knowledge additionally on a set $V$ of new propositions,
  \begin{itemize}
  \item[] [\VarEnt] if the restriction of one
  knowledge base on its original signature is at most as strong as that of the other, i.e.
  $\Forget(\Pi,V)\models\Forget(\Sigma,V)$.
  \item[] [\VarEQ] if the restriction of the two
  knowledge bases on a common signature are equivalent, i.e.
  $\Forget(\Pi,V)\equiv\Forget(\Sigma,V)$.
  \end{itemize}
\end{itemize}

We answer these problems for CNF, DNF, Horn, ren-Horn, q-Horn, and Krom fragments of propositional logic.
The main complexity results are summarized in  Table~\ref{tab:complexity}, from which one can see
that for Krom (resp. DNF) fragments, all of the six problems are tractable (resp. \coNP-complete).
While comparing Horn and its variants with CNF fragments, the corresponding complexity of the
former is one level below the latter in the complexity hierarchy.

\begin{table}[t]
  \centering
  \caption{Complexity results}\label{tab:complexity}
\begin{tabular}{ l |l| l| c| c}
            & CNF         & Horn/ren-Horn/q-Horn & Krom & DNF\\
  \hline
  \VarEQ    & \PIP{2}-c & \coNP-c   & P & \coNP-c  \\
  \VarIND   & \coNP-c   & P         & P & \coNP-c  \\
  \VarWeak  & \PIP{2}-c & \coNP-c   & P & \coNP-c  \\
  \VarStrong& \coNP-c   & P         & P & \coNP-c  \\
  \VarMatch & \PIP{2}-c & \coNP-c   & P & \coNP-c  \\
  \VarEnt   & \PIP{2}-c & \coNP-c   & P & \coNP-c\\
\end{tabular}
\end{table}

%Informally speaking, the result of forgetting a variable from a theory is a theory mentioning not the variable any more.
%In different forgetting contexts, there are some basic principles, such as Zhang-Zhou's  postulates for knowledge
%forgetting in modal logic S5 \cite{Yan:AIJ:2009}, which is a monotonic logical system; and the postulates for the forgetting in
%logic programs based on answer set semantics \cite{Yisong:IJCAI:2013}, which is a nonmonotonic logical system.

The rest of the paper are organized as follows. The basic notations of propositional logics and its fragments
are briefly introduced in Section~\ref{sec:preliminary}. Forgetting and its basic properties, algorithms and complexity
are presented in Section~\ref{sec:forget}.
%We consider applications of forgetting in other logical notions, including
%uniform interpolation, strongest necessary and weakest sufficient conditions,
%and definition in Section~\ref{sec:application}.
Related work and concluding remarks are discussed in Section~\ref{sec:application}
and \ref{sec:related-work} respectively. %. Finally,
%concluding remarks and future work are discussed in Section~\ref{sec:conclusion}.

\renewcommand{\thefootnote}

\section{Preliminaries}\label{sec:preliminary}
We assume a underlying propositional language ${\cal L}_{\cal A}$ with a finite set ${\cal A}$ of \textit{atoms}, called the
\emph{signature} of $\LA$.
A \textit{literal} is either an atom $p$ (called {\em positive literal}) or its \textit{negation} $\neg p$
(called {\em negative literal}). The \textit{complement} of a literal $l$ is $\neg l$.
The \emph{formulas} (of ${\cal L}_{\cal A}$) are defined as usual using connectives
$\land,\lor,\supset,\lrto$ and $\neg$. We assume two propositional constants $\top$ and $\bot$ for
tautology and contradiction respectively.
A \emph{theory} is a finite set of formulas. For a theory $\Sigma$, we use the following denotations:
\begin{itemize}
  \item $\neg\Sigma=\{\neg\var|\var\in\Sigma\}$,
  \item $\bigwedge\Sigma=\bigwedge_{\varphi\in\Sigma}\varphi$,
  \item $\bigvee\Sigma = \bigvee_{\varphi\in\Sigma}\varphi$, and
  \item $\Var(\Sigma)$ stands for the set of all atoms occurring in $\Sigma$.
\end{itemize}

An \emph{interpretation} is a set of atoms, which assigns \emph{true} to the atoms in
the set and \emph{false} to the others. The notion of \emph{satisfaction} between an
interpretation $I$ and a formula $\var$, written $I\models\var$, is
inductively defined in the standard manner. In this case $I$ is a called \emph{model} of $\var$.
By $\Mod(\varphi)$ we denote the set of models of $\varphi$.

A formula $\psi$ is a \emph{logical consequence} of  a formula $\varphi$, denoted by $\varphi\models\psi$, if
$\Mod(\varphi) \subseteq \Mod(\psi)$.
Two formulas $\var$ and $\psi$ are \emph{equivalent}, written $\var\equiv\psi$, if $\var\models\psi$ and $\psi\models\var$.
A formula $\psi$ is {\em irrelevant to} a set $V$ of atoms, denoted by $\IR(\psi,V)$, if there is a formula $\var$ such that
$\psi\equiv\var$ and $\Var(\var)\cap V=\emptyset$. Otherwise, $\psi$ is \emph{relevant}\footnote[1]{The definition of relevant is
equivalent with, but slightly different from, that of \cite{Lakemeyer:AIJ:1997}, in which $\psi$ is
\emph{relevant} to $V$ if there is a prime implicate of $\psi$ which mentions some atom from $V$.  } to $V$.

\subsection{Clauses and terms}
In the following we assume that $\neg\neg\var$ is shortten to $\var$ where $\var$ is a formula, unless explicitly stated otherwise.
A {\em clause} is an expression of the form $l_1\lor\cdots\lor l_n~(n\ge 0)$ where
$l_i~(1\le i\le n)$ are literals such that  $l_i\neq\neg l_j$ for every $i,j~(1\le i< j\le n)$. It is an
\emph{empty clause} in the case $n=0$, which means \textit{false}.
Dually,  a {\em term} is an expression of the form $l_1\land\cdots\land l_n~(n\ge 0)$ where
$l_i~(1\le i\le n)$ are literals such that $l_i\neq\neg l_j$ for every $i,j~(1\le i< j\le n)$.
By abusing the notation, we identify a clause $l_1\lor\cdots\lor l_n$ and
a term $l_1\land\cdots\land l_n$ with the set $\{l_1,\ldots,l_n\}$ when it is clear from its context.

A {\em conjunctive normal form (CNF)} formula is a conjunction of clauses, and a
{\em disjunctive normal form (DNF)} formula is a disjunction of terms. A $k$-CNF (resp. $k$-DNF) formula is a CNF (resp. DNF)
formula whose each clause (resp. term) contains no more than $k$ literals. In particular, $2$-CNF formulas are called \emph{Krom} formulas \cite{Krom:JSL:1970}.

A \textit{prime implicate} of a formula $\varphi$ is a clause $c$ such that $\varphi\models c$ and $\varphi\not\models c'$ for
every proper \emph{subclause} $c'\subset c$. Dually, a \textit{prime implicant} of $\varphi$ is a term $t$ such that
$t \models\varphi$ and $t'\not\models\varphi$ for each proper \emph{subterm} $t'\subset t$. A CNF (resp. DNF) formula is \textit{prime},
if it contains only prime implicates (resp. implicants). By $\PI(\psi)$ (resp. $\IP(\psi)$) we denote
the set of prime implicates (resp. implicants) of formula $\psi$.

In the following we shall identify a theory $\Sigma$ with the formula $\bigwedge\Sigma$ when there is no confusion.
The following lemma is well-known \cite{Marquis:1999:handbook}.
\begin{lemma}\label{lem:prime:imp}
Let $\Sigma$ be a theory and $\var$ be a term. Then
\begin{enumerate}[(1)]
  \item $\bigwedge\PI(\Sigma)\equiv\bigvee \IP(\Sigma)\equiv\Sigma$.
  \item $\var$ is a prime implicant of $\Sigma$ iff $\neg\var$ is a prime implicate   of $\neg\Sigma$.
  \item If $\Pi\equiv\Sigma$ then $\PI(\Sigma)=\PI(\Pi)$ and $\IP(\Sigma)=\IP(\Pi)$.
\end{enumerate}
\end{lemma}

Two clauses $c,c'$ are \emph{resolvable}, if there is an atom $p$ such that $p,\neg p\in c\cup c'$
and $c^*=(c\cup c')\setminus \{p,\neg p\}$ is a legal clause, viz, $c^*$ contains no pair of complement
literals. In this case we denote the clause $c\cup c'\setminus \{p,\neg p\}$  by $\res(c,c')$,
which is called their {\em resolvent}; otherwise, $\res(c, c')$ is undefined. It is well-known
that all prime implicates of a CNF formula $\varphi$ can be generated by resolution.

\subsection{Horn formulas and its variants}

In the following, by $\Pos(c)$ (resp. $\Neg(c)$) we denote the set of atoms
occurring positively (resp. negatively) in the clause or term $c$. In this sense
a clause $c$ can be written as $\Pos(c)\cup\neg\Neg(c)$.

A clause $c$ is \emph{Horn} \cite{Horn:JSL:1951} if $|\Pos(c)|\le 1$. Here $|D|$ denotes the number of elements in the set $D$.
A \emph{Horn formula} is a conjunction of Horn clauses. A formula $\var$ is \emph{Horn expressible} if there
is a Horn formula $\psi$ such that $\psi\equiv\var$.
A Horn formula $\var$ is \emph{double Horn} \cite{Eiter:IC:1998} if
there is a Horn formula $\psi$ such that $\psi\equiv \neg\var$, i.e., the negation of $\var$ is also Horn expressible.

Given a formula $\var$ and $V\subseteq\cal A$, we denote $\ren(\var,V)$ the result of
replacing every occurrence of atom $p\in V$  in $\var$ by $\neg p$ and
$\neg\neg p$ is shortened to $p$. For instance $\ren(p_1\lor\neg p_2\lor \neg p_3, \{p_1,p_2\})$ is the formula $\neg p_1\lor p_2\lor \neg p_3$.
A CNF formula $\var$ is \emph{Horn renamable} \cite{Lewis:JACM:1978} iff there exists a Horn renaming for it,
i.e., $\ren(\var,V)$ is a Horn formula for some $V\subseteq\cal A$.

\begin{definition}[\cite{Boros:AMAI:1990,Lang:AIJ:2008}]\label{def:QH-partition}
A CNF theory $\Sigma$ has a \emph{QH-partition} iff there exists a partition $\{Q,H\}$ of $\Var(\Sigma)$ s.t for every clause
$\delta$ of $\Sigma$, the following conditions hold:
\begin{enumerate}[(i)]
  \item $|\Var(\delta)\cap Q|\le 2$.
  \item $|\Pos(\delta)\cap H|\le 1$.
  \item If $|\Pos(\delta)\cap H|=1$ then $\Var(\delta)\cap Q=\emptyset$.
%  \item $\delta$ contains no more than two variables from $Q$;
%  \item $\delta$ contains at most one positive literal from $H$;
%  \item if $\delta$ contains a positive literal from $H$, then it contains no variable from $Q$.
\end{enumerate}
\end{definition}
A CNF theory $\Sigma$ is \emph{q-Horn}  iff there exists a \emph{q-Horn renaming}
for it \cite{Boros:AMAI:1990}, i.e., there is a set $V\subseteq\cal A$
such that replacing in $\Sigma$ every occurrence of $p\in V$ by $\neg p$ leads to a CNF theory having a QH-partition $\{Q,H\}$. Here $\neg\neg p$
is shorten to $p$. It is not difficult to see that, every Horn theory is Horn renamable, every Horn renamable theory is q-Horn ($Q=\emptyset$),
and every 2-CNF theory is also q-Horn ($H=\emptyset$).
A CNF formula $\var$ is Krom (resp. ren-Horn and q-Horn) \emph{expressible} if there is Krom (resp.
ren-Horn and q-Horn) formula $\psi$ such that $\var\equiv\psi$.

In terms of Lemma~\ref{lem:prime:imp}, the following lemma are well-known.

\begin{lemma}\label{lem:Horn:expres:PI}
 Let $\Sigma$ be a CNF theory. The following conditions are equivalent.
 \begin{enumerate}[(i)]
   \item $\Sigma$ is Horn expressible.
   \item $\PI(\Sigma)$ is a Horn theory.
   \item $M_1\models\Sigma$ and $M_2\models\Sigma$ imply $M_1\cap M_2\models\Sigma$, i.e. $\Mod(\Sigma)$ is \emph{closed under intersection}.
 \end{enumerate}
\end{lemma}
%\begin{proof}
%  (i) $\LRto$ (iii) is well-known. We show (i) $\LRto$ (ii) in the following.
%
%  (i) $\Rto$ (ii).  It is trivial when $\Sigma\equiv\bot$ or $\Sigma\equiv\top$. We consider the other case.
%  Suppose there exists a clause $c\in \PI(\Sigma)$ with $c=p_1\lor\cdots\lor p_m\lor \neg q_1\lor\cdots\lor \neg q_n$
%  and $m\ge 2$ where $p_i~(1\le i\le m)$ and $q_j~(1\le j\le n)$ are atoms.
%  Let $M$ be the least model of $\Sigma$. By $M\models c$ and $c$ is a prime implicate of $\Sigma$, we have that
%  $M\cap\{p_1,\ldots,p_m\}=\emptyset$ and $n\ge 1$. Let ${\cal M}_1=\{M|M\models\Sigma\textrm{ and }\{q_1,\ldots,q_n\}\subseteq M\}$ and
%  ${\cal M}_2=\{M|M\models\Sigma\textrm{ and }\{q_1,\ldots,q_n\}\setminus M\neq\emptyset\}$. It is evident that
%  $\Mod(\Sigma)={\cal M}_1\cup {\cal M}_2$. Note that $\Mod(\Sigma)$ is closed under intersection.
%  It follows that there exists $p_i\in \bigcap{\cal M}_1$ for some $i~(1\le i\le m)$
%  due to $\bigcap{\cal M}_1\models c$. Thus $M\models p_i\lor \neg q_1\lor\cdots
%  \lor \neg q_n$  for any $M\in{\cal M}_1\cup{\cal M}_2$. This conflicts with $c\in \PI(\Sigma)$.
%
%
%  (ii) $\Rto$ (i). It follows from the fact $\PI(\Sigma)\equiv\Sigma$.
%\end{proof}

It is known that it is tractable to recognize if a CNF theory is ren-Horn \cite{Lewis:JACM:1978,Chandru:AMAI:1990},
q-Horn \cite{Boros:DAM:1994}, or double-Horn \cite{Eiter:IC:1998}, %\footnote{If $\Sigma$ is an arbitrary DNF formula then it is coNP-complete.}.
and the satisfiability of ren-Horn, q-Horn and double Horn formulas are all tractable.

\begin{proposition}\label{prop:ren:q:Horn:res}
   Let $\Sigma$ be a CNF theory, $V\subseteq\cal A$ and $c_1,c_2$  two resolvable clauses of $\Sigma$. We have the following.
\begin{enumerate}[(i)]
  \item $\res(\ren(c_1,V),\ren(c_2,V))=\ren(\res(c_1,c_2),V)$.
  \item If two subsets $Q,H$ of $\cal A$ with $Q\cap H=\emptyset$ and $\Var(c_1\cup c_2)\subseteq Q\cup H$
  satisfy the conditions (i), (ii) and (iii)
  of Definition~\ref{def:QH-partition} for both $c_1$ and $c_2$, then $Q$ and $H$ satisfy the same conditions
  for $\res(c_1,c_2)$ as well.
  %\item If $\{c_1,c_2\}$ have a QH-partition $\{Q,H\}$ then $\res(c_1,c_2)$ has a QH-partition $\{Q\setminus\{p\},H\setminus\{p\}\}$
  %where $p,\neg p\in c_1\cup c_2$.
  %If $\ren(\Sigma,V)$ is a Horn formula then $\ren(\PI(\Sigma),V)$ is a Horn formula.
\end{enumerate}
\end{proposition}
\begin{proof}
  Without loss of generality, suppose $c_1=\{p\}\cup c_1'$ and $c_2=\{\neg p\}\cup c_2'$.

  (i) Note that $\res(c_1,c_2)=c_1'\cup c_2'$, $\ren(c_1,V)=\ren(p,V)\cup \ren(c_1',V)$ and
  $\ren(c_2,V)=\ren(\neg p,V)\cup \ren(c_2',V)$. Due to the fact
  that $\ren(c_1,v)$ and  $\ren(c_2,V)$ are resolvable and $\res(\ren(c_1,V),\ren(c_2,V))=\ren(c_1',V)\cup\ren(c_2',V)$,
  it follows that   $\res(\ren(c_1,V),\ren(c_2,V))=\ren(\res(c_1,c_2),V)$.

  (ii) We consider the following two cases:

  (a) $p\in Q$.  We have the following:
  \begin{itemize}
    \item Note that  $p\in\Var(c_1)\cap\Var(c_2)$ and $|\Var(c_i)\cap Q|\le 2$ for $i=1,2$ by the condition (i)
    in Definition~\ref{def:QH-partition}. It shows
    that $|\Var(c_1'\cup c_2')\cap Q|\le 2$;
    \item By $\Var(c_i)\cap Q\neq\emptyset$ for $i=1,2$ we have that
    $|\Pos(c_i)\cap H|=0$ due to the fact $|\Pos(c_i)\cap H|\le 1$ and $|\Pos(c_i)\cap H|\neq 1$ according
    to the conditions (ii) and (iii) of Definition~\ref{def:QH-partition}.
    It follows $|\Pos(c_1'\cup c_2')\cap H|=0$.
  \end{itemize}

  (b) $p\notin Q$ i.e. $p\in H$. Now we have the following:
  \begin{itemize}
    \item Since $p\in \Pos(c_1)\cap H$ we have $\Var(c_1)\cap Q=\emptyset$ by conditions (ii) and (iii)
    of Definition~\ref{def:QH-partition}. It implies that
    $|\Var(c_1\cup c_2)\cap Q|=|Var(c_2)\cap Q|\le 2$ by condition (i) of Definition~\ref{def:QH-partition}.
    Thus $|\Var(c_1'\cup c_2')\cap Q|\le 2$.

    \item Note that $|\Pos(c_1)\cap H|\le 1$ by condition (ii) of Definition~\ref{def:QH-partition}
    and $p\in \Pos(c_1)\cap H$. It shows that $|\Pos(c_1)\cap H|=1$ and
    $\Pos(c_1')\cap H=\emptyset$, thus $|\Pos(c_1'\cup c_2')\cap H|\le 1$ due to $|\Pos(c_2)\cap H|\le 1$
    by condition (ii) of Definition~\ref{def:QH-partition}.

    \item In the case $|\Pos(c_1'\cup c_2')\cap H|=1$ we have that
    $|\Pos(c_2')\cap H|=1$ due to $\Pos(c_1')\cap H=\emptyset$, which shows that $|\Pos(c_2)\cap H|=1$ by condition (ii) of
    Definition~\ref{def:QH-partition}, and then $\Var(c_2)\cap Q=\emptyset$.  Recall that $\Var(c_1)\cap Q=\emptyset$ (see
    the proof in the first item). Thus $\Var(c_1\cup c_2)\cap Q=\emptyset$, then $\Var(c_1'\cup c_2')\cap Q=\emptyset$.
  \end{itemize}
It completes the proof.
\end{proof}

Let $\Sigma$ be a CNF theory. We define
\begin{align*}
 &\res^0_\Sigma=\Sigma,\\
 &\res^{n+1}_\Sigma=\res^n_\Sigma\cup\{\res(c,c')|c,c'\in \res^n_\Sigma\textrm{ and $c,c'$ are resolvable}\}.
\end{align*}

\begin{theorem}\label{thm:ren:q:horn:res}
  Let $V\subseteq\cal A$ and $\Sigma$ a CNF theory.
  \begin{enumerate}[(i)]
  \item If $\ren(\Sigma,V)$ is a Horn theory then $\ren(\res^n_\Sigma,V)$ is a Horn theory for $n\ge 0$.
  \item If the partition $\{Q,H\}$ of $\Var(\ren(\Sigma,V))$ satisfies the conditions (i), (ii) and (iii)
  of Definition~\ref{def:QH-partition} for every clause of $\Sigma$, then
  $\{Q,H\}$ satisfies the same conditions for every clauses in $\res^n_{\ren(\Sigma,V)}$ for $n\ge 0$.
  %\item If $\ren(\Sigma,V)$ has a QH-partition $\{Q,H\}$ then $\res^n_{\ren(\Sigma,V)}$ has a QH-partition $\{Q,H\}$ for $n\ge 0$.
  \end{enumerate}
\end{theorem}
\begin{proof}
 We prove the theorem by induction on $n$.\\
  (i)  Base: it trivially holds for $n=0$ due to $\res^0_\Sigma=\Sigma$.

  Step: Suppose that $\ren(\res^n_\Sigma,V)$ is a Horn formula. For any $c\in \res^{n+1}_\Sigma\setminus \res^n_\Sigma$,
  $c=\res(c_1,c_2)$ for some clauses $c_1,c_2$ of $\res^n_\Sigma$. According to (i) of Proposition~\ref{prop:ren:q:Horn:res} we have
  $\ren(c,V)=\ren(\res(c_1,c_2),V)=\res(\ren(c_1,V),\ren(c_2,V))$. It follows that $\ren(c,V)$ is a Horn clauses since
  the resolvent of two Horn clauses is a Horn clause.

  (ii) Base: it trivially holds for $n=0$ due to $\res^0_{\ren(\Sigma,V)}=\ren(\Sigma,V)$.

  Step: Suppose that $Q$ and $H$ satisfy the same conditions for every clauses in $\res^n_{\ren(\Sigma,V)}$.
  For every clause $c\in \res^{n+1}_{\ren(\Sigma,V)}\setminus
  \res^n_{\ren(\Sigma,V)}$, there are two resolvable clauses $c_1,c_2\in \res^n_{\ren(\Sigma,V)}$ such that
  $c=\res(c_1,c_2)$. In terms of (ii) of Proposition~\ref{prop:ren:q:Horn:res},
  $Q$ and $H$ satisfy the conditions (i), (ii) and (iii) of Definition~\ref{def:QH-partition} for
  the clause $c$. Thus $\{Q,H\}$ satisfies the same conditions for every clauses in $\res^{n+1}_{\ren(\Sigma,V)}$.
\end{proof}

Together with Lemma~\ref{lem:Horn:expres:PI} and the fact that
$|\res(c_1,c_2)|\le 2$ if $|c_i|\le 2~(1\le i\le 2)$, the theorem above implies:
\begin{corollary}\label{cor:Horn:variant:PI}
   Let $V\subseteq\cal A$ and $\Sigma$ a CNF theory. If $\Sigma$ is a Horn (resp. ren-Horn  and q-Horn) theory
   then $\PI(\Sigma)$ is a  Horn (resp.  ren-Horn  and q-Horn) theory.%, and if $\PI(\Sigma)$ is
%   a double Horn theory then $\Sigma$ is.%
%  \begin{enumerate}[(i)]
%  \item If $\Sigma$ is a Horn formula then $\PI(\Sigma)$ is a Horn formula.
%  \item If $\Sigma$ is double Horn if and only if $\PI(\Sigma)$ is double Horn.
%  \item If $\Sigma$ is Horn renamable then $\PI(\Sigma)$ is Horn renamable.
%  \item If $\Sigma$ is a q-Horn formula then $\PI(\Sigma)$ is a q-Horn formula.
%  \end{enumerate}
\end{corollary}

%It is evident that if $\PI(\Sigma)$ is a double Horn theory then $\Sigma$ is a double Horn theory. However,
As illustrated by the following example, the reverse of
the above corollary do not generally hold even if $\Sigma$ is Horn expressible.
\begin{example}
  Let $\Sigma=(p\lor q)\land (\neg p\lor\neg q)\land (p\lor \neg q)$. Since $\Mod(\Sigma)=\{\{p\}\}$ (over
  the signature $\{p,q\}$), $\Sigma$ is Horn expressible but it is not a Horn formula.
  In fact, $\PI(\Sigma)=\{p,\neg q\}$, which is a Horn theory. However
  $\Sigma$ is not Horn renamable as we have that $\ren(\Sigma,V)$ is not a Horn formula for any $V\subseteq\{p,q\}$.

  Let $\Pi=(p\lor q\lor r)\land (p \lor q\lor \neg r)\land(\neg p\lor \neg q\lor r)\land (\neg p\lor \neg q\lor \neg r)\land(p\lor\neg q)$.
  We have that $\PI(\Pi)=\{p,\neg q\}$. %\{p\lor q, \neg p\lor \neg q\}$.
  It is evident that
  $\PI(\Pi)$ is a 2-CNF formula, thus a q-Horn formula. However, one can verify that
  $\Pi$ is not a q-Horn formula.
  \Eed
\end{example}

Let $M,X$ be two sets of atoms. We denote $M\div X$ the symmetric difference $(M\setminus X)\cup (X\setminus M)$.
For a collection ${\cal M}$ of interpretations, we denote ${\cal M}\div X=\{M\div X|M\in\cal M\}$.
\begin{proposition}\label{prop:ren:Model}
  Let $\Sigma$ be a formula and $V\subseteq \cal A$. Then
  $\Mod(\Sigma)\div V=\Mod(\ren(\Sigma,V))$.
\end{proposition}
\begin{proof}
  $(\Rto)$ Let $M\in\Mod(\Sigma)\div V$. There exists $M'\models\Sigma$ such that $M=(M'\setminus V)\cup (V\setminus M')$.
  Suppose $M\not\models \ren(\Sigma,V)$. It follows that $M\not\models \ren(c,V)$ for some clause $c\in \Sigma$. By
  $M'\models c$ we have that $M'\models l$ for some literal $l$ in $c$. Evidently, if $\Var(l)\notin V$ then
  $l$ is also a literal of $\ren(c,V)$ and $M\models l$, thus $M\models \ren(c,V)$. In the
  case $\Var(l)\in V$, we consider the two cases, where $p$ is an atom:
  \begin{itemize}
    \item $l=p$. It shows that $p\in M'$ and then $p\notin M$. Thus $M\models \ren(c,V)$ due to $M\models\neg p$.
    \item $l=\neg p$. It shows $p\notin M'$ and then $p\in M$. Thus $M\models \ren(c,V)$ due to $M\models p$.
  \end{itemize}
  Either of the above two cases result in a confliction.

  $(\Lto)$ Let $M\in \Mod(\ren(\Sigma,V))$. We have that
  $(M\setminus V)\cup (V\setminus M)\models \ren(\ren(\Sigma,V),V)$, which implies
  $(M\setminus V)\cup (V\setminus M)\models \Sigma$, i.e. $M\in\Mod(\Sigma)$.
\end{proof}

The following corollary easily follows from the proposition above.

\begin{corollary}
  Let $\Sigma$ be a CNF theory.  Then
  $\Sigma$ is Horn renamable iff there exists $V\subseteq\cal A$ such that
  $\Mod(\Sigma)\div V$ is closed under intersection.
\end{corollary}

\section{Forgetting}\label{sec:forget}
Starting with the basic notations and properties of forgetting, we will consider a general algorithm for
computing forgetting results of CNF theories, and computational complexity on various
reasoning problems relating
to forgetting.

Let $\Sigma$ be a propositional formula, we denote $\Sigma[p/\top]$ (resp. $\Sigma[p/\bot]$) the formula obtained from $\Sigma$ by substituting
all occurrences of $p$ with $\top$ ({\em true}) (resp. $\bot$ ({\em false})).  For
instance, if $\Sigma = \{p\supset q, (q \wedge r) \supset s\}$, then
$\Sigma[q/\top] \equiv \{r \supset s\}$ and $\Sigma[q/\bot] \equiv \{\neg
p\}$.

\subsection{Basic properties}

Let $M,N$ be two interpretations and $V\subseteq\cal A$.
$M$ and $N$ are $V$-\emph{bisimilar}, written $M\sim_VN$, if and only if $M\setminus V=N\setminus V$.
\begin{definition}[\cite{Fangzhen:forgetit}]
  \label{def:forget}
  Let $\var$ be a formula and $V\subseteq\cal A$. A formula $\psi$ is a \emph{result of forgetting $V$ from $\var$}
  iff, for every model $M$ of $\psi$, $\var$ has a model $M'$ such that $M\sim_VM'$.
\end{definition}

The syntactic counterpart of forgetting is a binary operator, written $\Forget(.,.)$, which
is defined recursively as:
\begin{align*}
  \Forget(\var,\emptyset)&= \var,\\
  \Forget(\var,\{p\})&= \var[p/\top]\vee \var[p/\bot],\\
  \Forget(\var,V\cup\{p\})&=\Forget(\Forget(\var,\{p\}),V)
\end{align*}
where $\var$ is a formula and $V\subseteq\cal A$.

Due to the fact that
if $\var'$ and $\psi'$ is a result of forgetting $V$ from $\var$ and $\psi$ respectively, then
$\var'\equiv\psi'$, by abusing the notation, we will denote $\Forget(\var,V)$
the result of forgetting $V$ from $\var$ when there is no ambiguity.

 The following proposition easily follows from the definition of forgetting,
 cf, Propositions~17 and 21 of \cite{Lang:JAIR:2003}.

\begin{proposition}\label{prop:forget:1}
  Let $\psi,\phi$ be two formulas and $V\subseteq\cal A$. Then we have
  \begin{enumerate}[(i)]
    \item $\forget(\psi\vee\phi,V)\equiv\forget(\psi,V)\vee\forget(\phi,V)$.
    \item $\forget(\psi\wedge\phi,V)\equiv\forget(\psi,V)\wedge\phi$ if $\IR(\phi,V)$.
    %\item If $\var\equiv\psi$ then $\Forget(\var,V)\equiv\Forget(\psi,V)$.
  \end{enumerate}
\end{proposition}

To establish a semantic characterization of forgetting, we introduce the notion of extension. Let
$M$ be an interpretation and $V\subseteq\cal A$. The \emph{extension} of $M$ over $V$, written $M\Extension{V}$, is the collection $\{X\subseteq{\cal A}|X\sim_VM\}$. The
\emph{extension} of a collection $\cal M$ of interpretations is $\bigcup_{M\in{\cal M}}M\Extension{V}$.
The following lemma establishes the semantic characterization of the syntactic forgetting, which
says that $\var$ is a result of forgetting $V$ from $\psi$ if and only if the models of $\var$ consist of
the $V$-extensions of models of $\psi$.

The following proposition is a variant of Corollary~1 of \cite{Lang:JAIR:2003}
and an extension of Corollary~5 of  \cite{Lang:JAIR:2003}.
\begin{proposition} Let $\var,\psi$ be two formulas and $X\subseteq\cal A$. Then
  $\var\equiv\Forget(\psi,V)$ if and only if $\Mod(\var)=\Mod(\psi)\Extension{V}$.
\end{proposition}
\begin{proof}
  $(\Rto)$  On the one hand, for every $M\in\Mod(\var)$, there exists $M'\in\Mod(\psi)$ such
  that $M\sim_VM'$ by Definition~\ref{def:forget}, i.e. $M\in\Mod(\psi)\Extension{V}$. On the other hand,
  if $M\in\Mod(\psi)\Extension{V}$ then there exists $M'\in\Mod(\psi)$ such that $M\sim_VM'$, which
  shows that $M\models\var$ by Definition~\ref{def:forget} again. Thus $\Mod(\var)=\Mod(\psi)\Extension{V}$.

  $(\Lto)$ Note that $\Mod(\var)=\Mod(\psi)\Extension{V}$ implies, for every $M\models\var$, there
  exists a mode $M'\models\psi$ such that $M\sim_VM'$. Thus $\var$ is a result of
  forgetting $V$ from $\psi$ by Definition~\ref{def:forget}, i.e. $\var\equiv\Forget(\psi,V)$.
\end{proof}

%
%\begin{proposition}\label{prop:forget:1}
%  Let $\psi,\phi$ be two formulas and $V$ a set of atoms. Then we have
%  \begin{enumerate}[(i)]
%    \item $\forget(\psi\vee\phi,V)\equiv\forget(\psi,V)\vee\forget(\phi,V)$.
%    \item $\forget(\psi\wedge\phi,V)\equiv\forget(\psi,V)\wedge\phi$ if $\IR(\phi,V)$.
%  \end{enumerate}
%\end{proposition}

The following theorem shows that the forgetting is closely connected with prime implicates and implicants.
\begin{theorem}\label{thm:forget:1}
  Let $\Pi,\Sigma$ be two theories and $V$ a set of atoms. The following conditions are equivalent to each other.
  \begin{enumerate}[(i)]
    \item $\Sigma\equiv\Forget(\Pi,V)$.
    \item $\Sigma\equiv\{\psi|\Pi\models\psi\ \mbox{and}\ \IR(\psi,V)\}$.
    \item $\Sigma\equiv \bigvee\{t|t\in\IP(\Pi)\ \textrm{and } \Var(t)\cap V=\emptyset\}$.
    \item $\Sigma\equiv\{c|c\in\PI(\Pi)\ \textrm{and}\ \Var(c)\cap V=\emptyset\}$.
  \end{enumerate}
\end{theorem}
\begin{proof}
 (i) $\LRto$ (ii). It is trivial if $\Pi\equiv\bot$. Suppose $\Pi$ is not falsity.
 Let $\Pi'=\{\psi|\Pi\models\psi\ \mbox{and}\ \IR(\psi,V)\}$. It is sufficient to prove $\Forget(\Pi,V)\equiv\Pi'$.
 On the one side, $M\models\Forget(\Pi,V)$ implies $\exists M'\models\Pi$ such that $M\sim_VM'$.
 It follows that $M'\models \Pi'$. On the other side, $M'\models\Pi'$ implies $M'$ can be modified
 to a model $M$ of $\Pi$ where $M\sim_VM'$. It shows that $M'\in\Mod(\Pi)\Extension{V}$.

 %It follows from the uniform interpolation property of propositional logic \cite{}.

 (i) $\LRto$ (iii). % We have that \\
 $\forget(\Sigma,V)$\\
 $\equiv\forget(\bigvee \IP(\Sigma),V)$ as $\Sigma\equiv\bigvee\IP(\Sigma)$\\
 $\equiv \bigvee_{t\in\IP(\Sigma)}\forget(t,V)$ by (i) of Proposition~\ref{prop:forget:1}\\
 $\equiv \bigvee\{t|t\in\IP(\Pi)\ \textrm{and } \Var(t)\cap V=\emptyset\}$ by (ii) of Proposition~\ref{prop:forget:1}.

 (i) $\LRto$ (iv). It is proved by Theorem~37 of \cite{Lakemeyer:AIJ:1997},
  and can follows from Propositions~19 and 20 of \cite{Lang:JAIR:2003}.
 .%The direction from left to right is obvious by (ii).
% Let us consider the other direction. As $\Sigma\equiv\Sigma'=\{\psi|\Pi\models\psi\ \mbox{and}\ \IR(\psi,V)\}$,
% we may assume that every element of $\Sigma'$ is a clause. Now we have that
% $\forget(\Sigma,V)\equiv\forget(\Sigma',V)=\{c|c\in\Sigma'\ \textrm{and}\ \IR(c,V)\}$ by (ii) of Proposition~\ref{prop:forget:1}.
% It is clear that, for any clause $c\in\forget(\Sigma',V)$, there must exist a prime implicate
% $c'$ of $\Sigma$ such that $\IR(c',V)$ and $c'\models c$. Thus we have $\{c|c\in\PI(\Sigma)\ \textrm{and}\ \IR(c,V)\}\models\forget(\Sigma,V)$.
\end{proof}

Actually, (i)$\LRto$(iv) is mentioned as a fact in \cite{DBLP:Lin:AIJ:2001}, which states that $\forget(\Sigma,V)$ is
equivalent to the conjunction of prime implicates of $\Sigma$ that do not mention any propositions from $V$.
In terms of Corollary~\ref{cor:Horn:variant:PI} and the theorem above, we have the following corollary.
\begin{corollary}\label{cor:Horn:variant:forget}
  Let $\Sigma$ be a CNF theory and $V\subseteq \cal A$.
  If $\Sigma$ is a Horn (resp. Krom,  ren-Horn  and q-Horn) expressible
  then $\Forget(\Sigma,V)$ is a Horn (resp. Krom, ren-Horn  and q-Horn) expressible.
\end{corollary}

\subsection{A resolution-based algorithm}%

Given a set $\Pi$ of clauses and an atom $p$, the {\em unfolding} of $\Pi$ w.r.t. $p$, written $\wunfold(\Pi,p)$,
is the set of clauses obtained from $\Pi$ by replacing every clause $c\in\Pi$ such that $p\in\Pos(c)$ with the clauses
$$\res(c,c_i)~(1\le i\le k)$$
%of the form
%%\begin{equation*}\label{eq:unfold:1}
% $ A\lto B$ such that $p\in B$, with the rules
%\end{equation*} such that $p\in B$, with the rules,
%\begin{equation}\label{eq:unfold:2}
%   res(c,c_i)
%   %(A\cup A_i)\setminus \{p\}\lto (B_i\cup B)\setminus \{p\}~(1\le i\le n)
%\end{equation}
where $c_1,\ldots,c_k$ are all the clauses of $\Pi$ such that $p\in\Neg(c_i)$ and, the two clauses
$c$ and $c_i$ are resolvable for every $i~(1\le i\le k)$. %$(\Pos(c)\cup\Pos(c_i))\cap (\Neg(c)\cup\Neg(c_i))=\{p\}$.
%where $A_i\lto B_i$ are all the rules in $\Pi$ such that $p\in A_i$ and  $(A\cup A_i)\cap (B\cup B_i)=\{p\}$
%for every  $i~(1\le i\le n)$\footnote{The unfolding is called {\em disjunctive partial deduction},
%and the unfolding result is called a {\em residual program} in \cite{Sakama:JLP:1997}.}
In particular, if $k=0$ then  $\wunfold(\Pi,p)$ is obtained from $\Pi$ by simply removing all the clauses that contain the positive literal $p$. %
%If $|(A\cup A_i)\cap (B\cup B_i)|\ge 2$ then (\ref{eq:unfold:2}) is a tautology, which can be safely
%removed without loss any information. %It easily follows  that $\Pi$ and $\wunfold(\Pi,p)$ have same
%minimal models (cf. Theorem 3.2 of \cite{Sakama:JLP:1997}).

The {\em strong unfolding} of $\Pi$ w.r.t. an atom $p$, denoted $\unfold(\Pi,p)$, is obtained from
$\wunfold(\Pi,p)$ by removing all clauses containing $\neg p$.

\begin{example}
  Let us consider the below two  CNF theories.
  \begin{align*}
     &\Pi=\{p\vee q\lor\neg a, \qquad p\lor\neg q, \qquad b\lor\neg p, \qquad c\lor\neg p\}.\\
     &\Sigma=\{p\lor\neg a, \qquad p\lor\neg q\lor\neg b,\qquad q\lor\neg p, \qquad c\lor\neg p\}.
  \end{align*}
  We have that
  \begin{align*}
  ¡¡&\unfold(\Pi,p)=\{b\vee q\lor\neg a, \qquad c\vee q\lor\neg a, \qquad b\lor\neg q, \qquad c\lor\neg q\},\\
    &\unfold(\Pi,q)=\{p\lor\neg a, \qquad b\lor\neg p, \qquad c\lor\neg p\},\\
    &\unfold(\unfold(\Pi,p),q)=\{b\lor\neg a,\qquad c\lor\neg a, \qquad b\vee c\lor\neg a\},\\
    &\unfold(\unfold(\Pi,q),p)=\{b\lor\neg a, \qquad c\lor\neg a\},\\
    &\unfold(\Sigma,p)=\{q\lor\neg a,\qquad c\lor\neg a,\qquad c\lor\neg q\wedge b\},\\
    &\unfold(\Sigma,q)=\{p\lor\neg a,\qquad c\lor\neg p\},\\
    &\unfold(\unfold(\Sigma,p),q)=\{c\lor\neg a,\qquad c\lor\neg a\lor\neg b\},\\
    &\unfold(\unfold(\Sigma,q),p)=\{c\lor\neg a\}.
  \end{align*}
Though $\unfold(\unfold(\Pi,p),q)\neq \unfold(\unfold(\Pi,q),p)$, we will see that the two
theories are equivalent, i.e., having same models. \Eed %It is similar to the logic program $\Sigma$.
\end{example}

As demonstrated by Theorem \ref{thm:forget:1}, forgetting results always exist,
as every formula can be translated into an equivalent CNF theory.
The below proposition shows that forgetting in CNF theories can be achieved by unfolding.
\begin{theorem}\label{thm:forget:unfold}
  Let $\Pi$ be a CNF theory and $p\in\cal A$. Then $\Forget(\Pi,p)\equiv\unfold(\Pi,p)$.
\end{theorem}
\begin{proof}
  Without loss of generality, we assume that $\Pi$ contains no tautology.
  Note that if the clause $c: A\cup\neg B$ in $\Pi$ satisfies $p\notin A\cup B$ then $c\in\unfold(\Pi,p)$ and
  $\forget(\Pi,p)\models c$ by (ii) of Proposition \ref{prop:forget:1}. Thus we can  assume
  $p\in A\cup B$ for every clause $A\cup\neg B$ of $\Pi$.

  Let $c_i~(1\le i\le n)$
  be all the clauses of $\Pi$ such that $p\in c_i$, and $c_j'~(1\le j\le m)$
  be all the clauses of $\Pi$ such that $\neg p\in c_j'$.

  The direction from left to right is clear by (ii) of Theorem~\ref{thm:forget:1},
  i.e., $\forget(\Pi,p)\models\unfold(\Pi,p)$, since
  $\Pi\models \res(c_i,c_j')$
  for every $i,j~(1\le i\le n,1\le j\le m)$ whenever $c_i,c_j'$ are resolvable.

  To prove the other direction, it is sufficient to show that for every model $M$
  of $\unfold(\Pi,P)$, there exists a model $M'$ of $\Pi$ such that $M'\sim_p M$. We prove this by contradiction.
  Without loss of generality, let $M\models\unfold(\Pi,p)$, $p\notin M$, $M'=M\cup\{p\}$, $M\not\models\Pi$ and
  $M'\not\models\Pi$. It follows that $M\not\models c_i$ for some $i~(1\le i\le n)$ and
  $M'\not\models c_j'$ for some $j~(1\le j\le m)$.
  Let us consider the following two cases:

  (1) $c_i$ and $c_j'$ are not resolvable. It shows that there is an atom $q$ different from $p$ such that
  $q,\neg q\in c_i\cup c_j'$. Recall that $c_i,c_j'$ are not tautology.
  In the case $q\in M$  we have that $q\in c_j'$ and $\neg q\in c_i$ as $M\not\models c_i$.
  It shows that $M\models c_j'$, thus $M'\models c_j'$, a contradiction.
  In the case $q\notin M$ we have that $q\in c_i$ and $\neg q\in c_j'$ as $M\not\models c_i$.
  It follows that $M\models c_j'$, thus $M'\models c_j'$,  a contradiction.

  (2) $c_i$ and $c_j'$ are resolvable. It shows that the resolvent $\res(c_i,c_j')=(c_i\setminus \{p\})\cup(c_j'\setminus \{\neg p\})$
  belongs to $\unfold(\Pi,p)$. Note that $M\not\models c_i$ implies $M\not\models c_i\setminus \{p\}$.
  It follows that $M\models c_j'\setminus \{\neg p\}$ since $M\models \res(c_i,c_j')$,
  thus $M'\models c_j'\setminus \{\neg p\}$ and $M'\models c_j'$ by $c_j'\setminus \{\neg p\}\models c_j'$, a contradiction.
\end{proof}

\begin{proposition}\label{prop:sunfold:order:irrelevance}
  Let $\Pi$ be a CNF theory, $p,q$ two atoms. Then  we have that
  $$\unfold(\unfold(\Pi,p),q)\equiv\unfold(\unfold(\Pi,q),p).$$
\end{proposition}
\begin{proof}
  By Theorem \ref{thm:forget:unfold}, we have that\\
  $\unfold(\unfold(\Pi,p),q)$\\
  $\equiv\unfold(\forget(\Pi,p),q)$\\
  $\equiv\forget(\forget(\Pi,p),q)$\\
  $\equiv\forget(\Pi,\{p,q\})$\\
  $\equiv\forget(\forget(\Pi,q),p)$\\
  $\equiv\forget(\unfold(\Pi,q),p)$\\
  $\equiv\forget(\forget(\Pi,q),p)$\\
  $\equiv\unfold(\unfold(\Pi,p),q)$.
\end{proof}

In terms of the above proposition, the  unfolding is independent of the ordering of atoms to be strongly unfolded.
We define  unfolding a set of atoms as following,
\begin{align*}
  &\unfold(\Pi,\emptyset)=\Pi,\\
  & \unfold(\Pi,V\cup\{p\})=\unfold(\unfold(\Pi,p),V)
\end{align*}
where $\Pi$ is a CNF theory and $V\subseteq\cal A$.

It follows that, by Theorem~\ref{thm:forget:unfold} and Proposition~\ref{prop:sunfold:order:irrelevance},
\begin{corollary}\label{cor:forget:unfold:eq}
  Let $\Pi$ be a CNF theory and $V\subseteq\cal A$.
  $\forget(\Pi,V)\equiv\unfold(\Pi,V)$.
\end{corollary}

In terms of Corollaries~\ref{cor:Horn:variant:PI} and \ref{cor:forget:unfold:eq}, we have%
%the fact $\unfold(\Sigma,V)\subseteq\PI(\Sigma)$, we have the following corollary.
\begin{corollary}\label{cor:Horn:variant:unfold}
  Let $\Sigma$ be a CNF theory and $V\subseteq \cal A$.
  If $\Sigma$ is a Horn (resp. Krom,  ren-Horn  and q-Horn) theory
  then $\unfold(\Sigma,V)$ is a Horn (resp. Krom, ren-Horn  and q-Horn) theory.
\end{corollary}

%
%It also shows that Horn theories have the uniform interpolation property.
%\begin{corollary}
%  Let $\alpha$ be a Horn theory and $V$ a set of atoms. There exists a Horn theory $\gamma$ such that
%  $var(\gamma)\subseteq var(\alpha)\setminus V$, and for any formula $\beta$ with $var(\beta)\cap V=\emptyset$,
%  $\alpha\vdash\beta\ \textrm{iff}\ \gamma\vdash\beta.$
%\end{corollary}
%
%The below corollary follows from Theorem \ref{thm:forget:unfold}, (iii) and (iv) of Theorem \ref{thm:forget:1}.
%\begin{corollary}\label{cor:unfold:PI}
% For any logic program $\Pi$ and atom $p$, we have
%  \[\unfold(\Pi,p)\equiv\{c\in\PI(\Pi)|p\notin \Var(c)\}\equiv \bigvee_{t\in\IP(\Pi)}t\setminus \{p,\neg p\}.\]
%\end{corollary}

The strong unfolding provides alternative approach of evaluating forgetting. In particular, strong unfolding
results of CNF theories are in CNF as well. If $\Pi$ is a Horn theory then $\forget(\Pi,V)$ is also Horn which
can be achieved by strong unfolding. It distinguishes from the syntactic approach
$\forget(\Pi,p)=\Pi[p/\bot]\vee\Pi[p/\top]$, which is not in CNF,
though it can be transformed into CNF (with possibly much more expense).

Based on the notion of strong unfolding,
we present the algorithm for computing forgetting results of CNF theories in
Algorithm \ref{algo:forget:unfold}.
The following proposition asserts the correctness.
\begin{proposition}
  \label{prop:algo:correct}
  Let $\Pi,V,\Sigma$ be as in Algorithm~\ref{algo:forget:unfold}. Then $\Sigma\equiv\Forget(\Pi,V)$.
\end{proposition}
\begin{proof}
  It follows from that the lines 3-9 of Algorithm~\ref{algo:forget:unfold} compute $\Forget(\Pi,p)$.
\end{proof}

  \begin{algorithm}[t]
  \LinesNumbered
  \SetKwInOut{Input}{input}\SetKwInOut{Output}{output}
  \Input{A set $\Pi$ of clauses and a set $V$ of atoms}
  \Output{The result of forgetting $V$ in $\Pi$ }
  \Begin{
    $S\lto\{c|c\in\Pi\textrm{ and }V\cap\Var(c)=\emptyset\}$\;
    $\Pi\lto\Pi\setminus S$\;
    %$\Sigma\lto\emptyset$\;
    \ForEach{$(p\in V)$}{
    $\Pi'\lto \{c|c\in\Pi\textrm{ and }p\in\Var(c)\}$\;
    $\Sigma\lto \Pi\setminus\Pi'$\;
    %$\Pi''=\Pi\setminus\Pi'$\;
       \ForEach{$(c\in\Pi'$ s.t $p\in \Pos(c))$}{
         \ForEach{$(c'\in\Pi'$ s.t $p\in\Neg(c')$ and $c,c'$ are resolvable$)$}{
           %\If{$(\Sigma\not\models\res(c,c'))$}
           {
           $\Sigma\lto \Sigma\cup \res(c,c')$\;}
         }
       }
    $\Pi\lto\Sigma$\;
    }
    \Return{$\Sigma\cup S$}
    }
    \caption{An Algorithm for Forget($\Pi,V$)}
    \label{algo:forget:unfold}
  \end{algorithm}

The algorithm remains the potentiality of heuristics. For example, one can  forget
the atoms one by one in a specific order, and similarly choose two specific clauses to do resolution sequentially.
In addition, to save space, one can add the condition $\Sigma\not\models\res(c,c')$
at line 7 of the algorithm.  While checking the condition is intractable generally,
however, it is tractable for some special CNF theories, including Horn, ren-Horn, q-Horn and Krom ones.

Before end of the section, we formally analyze the computational costs.
\begin{proposition}
  Let $\Pi$ be a CNF theory and $V\subseteq\cal A$ where $|\Pi|=n$ and $|V|=k$. The
  time and space complexity of Algorithm~\ref{algo:forget:unfold} are $O(n^{2^k})$.
\end{proposition}
\begin{proof}
  It follows from that the lines 5-9 of the algorithm, which is to compute $\unfold(\Pi,p)$,
  is bounded by $O(|\Pi|^2)$, and the size of $\unfold(\Pi,p)$ is bounded by $O(|\Pi|^2)$ as well.
\end{proof}

One can evidently note that, if $k$ is given as a fixed parameter then $\unfold(\Pi,V)$
can be computed in polynomial time in the size of $\Pi$.
The following example shows that an exponential explosion of $\Forget(\Pi,V)$
is inescapable even if $\Pi$ is a Horn theory. % though
%Lakemeyer has shown that the problem of deciding if a Horn theory is relevant to a set of atoms is tractable
%(cf. Theorem 51 of \cite{Lakemeyer:AIJ:1997}).
\begin{example}\label{exam:8}
  Let $\Pi$ be the Horn theory consisting of
  \begin{align*}
    p\lor\neg q_1\lor\ldots\lor\neg q_n, \quad q_1\lor\neg r_1,\quad q_1\lor\neg r_1', \quad  \ldots, \quad q_n \lor\neg r_n, \quad q_n\lor\neg r_n'.
  \end{align*}
  It is not difficult to see that, for each subset $I$ of $N=\{1,\ldots, n\}$,
  $$\Pi\models \left(\bigvee_{i\in I}\neg r_i\right)\lor\left(\bigvee_{j\in (N\setminus I)}\neg r_j'\right)\lor p.$$
  Thus $\forget(\Pi,\{q_1,\ldots,q_n\})$ is in exponential size of $\Pi$ since
  there are $2^n$ number of subsets of $N$.
  And as a matter of fact, there is no Horn theory that is in polynomial size of $\Pi$ and is equivalent to $\forget(\Pi,\{q_1,\ldots,q_n\})$ since
  $\left(\bigvee_{i\in I}\neg r_i\right)\lor\left(\bigvee_{j\in (N\setminus I)}\neg r_j'\right)\lor p$ is a prime implicate of $\Pi$.
  \Eed
\end{example}

Note that, in the case $\Pi$ is a Krom theory, there are at most $O(m^2)$ number clauses where $m=|\Var(\Pi)|$.
Thus $|\Sigma|$ in the line 7 of Algorithm~\ref{algo:forget:unfold} is bounded by $O(n^2)$ where $n=|\Pi|$.
Then the overall time and space complexity is $O(kn^2)$ whenever $\Pi$ is a Krom theory where $k=|V|$.

\subsection{Complexities}\label{sub:sec:complexity}
In the following we consider the complexities of reasoning problems on forgetting for various fragments
of propositional logic.
\subsubsection{DNF, CNF and arbitrary theories}
\begin{proposition}
  Let $\Pi,\Sigma$ be two (CNF) theories, and $V\subseteq\cal A$. We have that
  \begin{enumerate}[(i)]
    \item deciding if $\Pi\models\forget(\Sigma,V)$ is \PIP{2}-complete,
    \item deciding if $\forget(\Pi,V)\models\Sigma$ is \coNP-complete,
    \item deciding if $\forget(\Pi,V)\models\forget(\Sigma,V)$ is \PIP{2}-complete.
  \end{enumerate}
\end{proposition}
\begin{proof}
  (i) Membership. In the case $\Pi\not\models\forget(\Sigma,V)$, there exists a model $M$ of $\Pi$ such that
  $M\not\models\forget(\Sigma,V)$, i.e. for every model $M'$ of $\Sigma$ such that $M\sim_VM'$, $M'\not\models\Sigma$, which
  can be done in polynomial time in the size of $\Sigma$ and $V$ by calling a nondeterministic Turing machine.

  Hardness. It follows from the fact that $\top\models\forget(\Sigma,V)$ iff $\forget(\Sigma,V)$ is valid,
  i.e. $\forall V'\exists V\Sigma$ is valid, where $V'=\Var(\Sigma)\setminus V$. The latter is \PIP{2}-complete even if
  $\Sigma$ is a CNF theory, as every formula can be translated into a CNF theory with auxiliary variables that preserves
  the satisfiability, informally $\forall V'\exists V\Sigma$ can be translated polynomially
  into $\forall V'\exists V\exists V^*\Sigma'$ such that (a) $\Sigma'$ is a CNF theory, and (b)
  $\forall V'\exists V\Sigma$ is valid iff $\forall V'\exists V\exists V^*\Sigma'$ is valid, where
  $V^*$ is the introduced auxiliary variables  \cite{Buning:Handbook:Complexity:2009}.

  (ii) Membership. If $\forget(\Pi,V)\not\models\Sigma$ then there exists two sets $M$ and $M'$ such that
  $M\models\Pi,M'\not\models\Sigma$ and $M\sim_VM'$. It is in polynomial time to guess such
  $M,M'$ and check the conditions $M\models\Pi,M'\not\models\Sigma$ and $M\sim_VM'$. Hence the problem is in \coNP.

  Hardness. $\forget(\Pi,V)\models\bot$ if and only if $\Pi\models \bot$, i.e. $\Pi$ has no model, which is \coNP-hard. Thus
  the problem is \coNP-complete.

  (iii) Membership. If $\forget(\Pi,V)\not\models\Forget(\Sigma,V)$ then there exist an interpretation $M$ such that
  $M\models\Forget(\Pi,V)$ but $M\not\models\Forget(\Sigma,V)$, i.e., there is $M'\sim_VM$ with
  $M'\models\Pi$ but $M''\not\models\Sigma$ for every $M''$ with $M''\sim_VM$. It is evident that
  guessing such $M,M'$ with $M\sim_VM'$ and checking $M'\models\Pi$ are feasible, while checking
  $M''\not\models\Sigma$ for every $M''\sim_VM$ can be done in polynomial time in the size of $V$ and $\Sigma$
  by call a nondeterministic  Turing machine. Thus the problem is in \PIP{2}.

  Hardness. It follows from (i) due to the fact that $\forget(\Pi,V)\models\forget(\Sigma,V)$ iff $\Pi\models\forget(\Sigma,V)$.
\end{proof}
%
%Even if $\Sigma$ is a CNF formula, the validness of $\forall X\exists Y.\Sigma$ is still $\Pi_2^P$-hard, from which the below corollary
%follows by the proposition above.
%
%\begin{corollary}
%  Let $\Pi,\Sigma$ be two CNF formulae, and $V$ a set of atoms. We have that
%  \begin{enumerate}[(i)]
%    \item deciding if $\Pi\models\forget(\Sigma,V)$ is $\Pi_2^P$-complete,
%    \item deciding if $\forget(\Pi,V)\models\Sigma$ is coNP-complete,
%    \item deciding if $\forget(\Pi,V)\models\forget(\Sigma,V)$ is $\Pi_2^P$-complete.
%  \end{enumerate}
%\end{corollary}

The proposition implies:
\begin{corollary}\label{cor:forget:eq}
 Let $\Pi,\Sigma$ be two (CNF) theories, and $V\subseteq\cal A$. Then
 \begin{enumerate}[(i)]
   \item deciding if $\Pi\equiv\forget(\Sigma,V)$ is \PIP{2}-complete,
   \item deciding if $\forget(\Pi,V)\equiv\forget(\Sigma,V)$ is \PIP{2}-complete, and
   \item deciding if $\forget(\Pi,V)\equiv\Pi$ is \coNP-complete.
 \end{enumerate}
\end{corollary}

In the case $\Pi$ is an arbitrary propositional formula, (ii) and (iii) of the corollary corresponds
to \textsc{var-equivalence} and \textsc{var-independence} in \cite{Lang:JAIR:2003}, in which it is proved
to be the same complexity as that of CNF theory case, respectively. Note that the inverse of item (iii)
is the relevance problem, i.e., if a formula $\Pi$ is relevant to $V$,
which is \NP-hard (cf. Theorem~50 of \cite{Lakemeyer:AIJ:1997}).

Recall that $\forget(\varphi,p)=\varphi[p/\top]\vee\varphi[p/\bot]$ for a given formula
$\varphi$ and an atom $p$. According to (i) of Proposition~\ref{thm:forget:1}, when $\var$ is a term $l_1\land\cdots \land l_n$,
$\forget(\var,V)$ is the term obtained from $\var$ by replacing $l_i~(1\le i\le n)$ with $\top$ if $\Var(l_i)\subseteq V$.
E.g. $\Forget(p\land\neg q,\{p\})\equiv \neg q$ and $\Forget(p\land\neg q,\{q\})\equiv q$. It implies
that if $\Pi$ is a DNF theory then $\forget(\Pi,V)$ can be computed in linear time
in the size of $\Pi$  by (i) of Proposition \ref{prop:forget:1}.

\begin{proposition}
  Let $\Pi,\Sigma$ be two DNF theories, and $V\subseteq\cal A$. The following problems are \coNP-complete:
  \begin{enumerate}[(i)]
    \item deciding if $\Pi\models\forget(\Sigma,V)$,
    \item deciding if $\forget(\Pi,V)\models\Sigma$,
    \item deciding if $\forget(\Pi,V)\models\forget(\Sigma,V)$.
  \end{enumerate}
\end{proposition}
\begin{proof}
  (i) Membership. It is obvious that if $\Pi\not\models\forget(\Sigma,V)$ then there exists a set $M$
  of atoms such that $M\models\Pi$ and $M\not\models\forget(\Sigma,V)$. As $\forget(\Sigma,V)$ is computable
  in polynomial time, the checking  $M\models\Pi$ and $M\not\models\forget(\Sigma,V)$ is feasible in polynomial time as well.
  Hence the problem is in \coNP.

  Hardness. Let $\Pi\equiv\top$.
  Note that $\top\models\forget(\Sigma,V)$ iff $\forget(\Sigma,V)$ is valid. As $\forget(\Sigma,V)$ is still a DNF theory whose validness
  is \coNP-hard, it shows that the problem is \coNP-hard as well.

  (ii) and (iii) can be similarly proved as that of (i).
\end{proof}

The proposition above implies
\begin{corollary}
  Let $\Pi,\Sigma$ be two DNF theories, and $V\subseteq\cal A$. The following problems are \coNP-complete.
  \begin{enumerate}[(i)]
    \item deciding if $\Pi\equiv\forget(\Sigma,V)$,
    \item deciding if $\forget(\Pi,V)\equiv\forget(\Sigma,V)$,
    \item deciding if $\forget(\Pi,V)\equiv\Pi$.
  \end{enumerate}
\end{corollary}

\subsubsection{Horn theories and its variants}
For a Horn formula $\Sigma$, its {\em dependency graph} is the directed graph $G(\Sigma)=(V,E)$, where $V=\cal A$ and
$(a_i,a_j)\in E$ iff there is a Horn clause $c\in\Sigma$ such that $\neg a_i\in c$ and $a_j\in c$. A Horn formula
$\Sigma$ is {\em acyclic} if $G(\Sigma)$ has no directed cycle.

\begin{theorem}\label{thm:complexity:Horn}
  Let $\Pi,\Sigma$ be  Horn (resp. ren-Horn and q-Horn) theories and $V\subseteq\cal A$.
  \begin{enumerate}[(i)]
    \item  The problem of deciding if $\Pi\models\forget(\Sigma,V)$ is \coNP-complete,
    even if $\Pi$ and $\Sigma$ are acyclic.
    \item  The problem of deciding if $\forget(\Pi,V)\models\Sigma$ is tractable.
    \item  The problem of deciding if $\forget(\Pi,V)\models\forget(\Sigma,V)$ is \coNP-complete, even if $\Pi$ and $\Sigma$ are acyclic.
  \end{enumerate}
\end{theorem}
\begin{proof}
  (i) Membership. Note that $\Pi\not\models\forget(\Sigma,V)$
  iff there is a prime implicate $c$ of $\Sigma$ such that $\Var(c)\cap V=\emptyset$
  and $\Pi\not\models c$, the latter holds iff $\Pi\cup\neg c$ has a model, where
   $\neg c=\{\neg l|l\mbox{ is a disjunct of $c$}\}$. In the case $\Pi$ is  q-Horn, $\Pi\cup \neg c$ is
   q-Horn and its satiability checking is tractable
   \cite{Boros:AMAI:1990}. One can guess such a prime implicate $c$ and check
   if $\Pi\not\models c$ in polynomial time in the size of $\Pi$ and $\Sigma$.  Thus the problem is in \coNP\ even
   if $\Pi,\Sigma$ are q-Horn theories.

  Hardness. Let $\gamma=c_1\wedge\cdots\wedge c_m$ be a 3CNF formula over atoms $x_1,\ldots,x_n$,
  where $c_i=l_{i,1}\vee l_{i,2}\vee l_{i,3}$. The below construction is quite similar to
  the one used in the proof of Theorem 4.1 \cite{Eiter:JACM:2007}. We introduce for each
  clause $c_i$ a new atom $y_i$, for each atom $x_j$ a new atom $x_j'$ (which intuitively
  corresponds to $\neg x_j$), and a special atom $z$. The Horn theory
  $\Pi=\{\neg x_i\vee \neg x_i'|1\le i\le n\}$ and $\Sigma$ contains $\Pi$ and additional the below clauses:
  \begin{align*}
    & \neg z\vee y_1,\\
    & \neg y_i\vee \neg l_{i,j}^*\vee  y_{i+1}\mbox{ for all $i=1,\ldots,m-1$, and $j=1,2,3$},\\
    & \neg y_m\vee \neg l_{m,j}^* \mbox{ for $j=1,2,3$}
  \end{align*}
  where $l^*=x$ if $l$ is a positive literal $x$, and $l^*=x'$ if $l$ is a negative literal $\neg x$. It is clear that
  both $\Pi$ and $\Sigma$ are acyclic Horn formulas, thus Horn renamable and q-Horn formulas.
  We claim that $\gamma$ is satisfiable iff $\Pi\not\models\forget(\Sigma,V)$ where $V=\{y_1,\ldots,y_m\}$.
  It is easy to see that $\Sigma$ has a prime implicate $c$ such that $\Var(c)\cap V=\emptyset$ and $c\notin \Pi$ iff $\Pi\not\models\forget(\Sigma,V)$.

  On the one hand, let $\sigma$ be a satisfying assignment of $\gamma$.  Then we arbitrarily choose from each $c_i$ a literal $l_{i,j_i}$ satisfied by $\sigma$. It  follows that $c=\neg z\vee (\bigvee_{1\le i\le m} \neg l^*_{i,j_i})$ is an implicate of $\Sigma$ where $j_i\in\{1,2,3\}$, and $c$ contains at most one literal in $\{\neg x_i,\neg x_i'\}$ for every $i~(1\le i\le n)$. As $\Var(c)\cap V=\emptyset$, and $\bigvee_i \neg l^*_{i,j_i}$ is not an implicate of $\Pi$
  since there is no subclauses of it is generated by the resolution procedure for $\Pi$, we have that $c$ is a
  prime implicate of $\Sigma$ and $\Pi\not\models c$. Thus $\Pi\not\models\forget(\Sigma,V)$.

  On the other hand,  there exists a prime implicate $c$ of $\Sigma$ such that
  both $\Pi\not\models c$ and $\Var(c)\cap V=\emptyset$ due to $\Pi\not\models\forget(\Sigma,V)$. This prime implicate $c$ can only be generated from the Horn clauses in $\Sigma\setminus \Pi$ and has the form
  $\neg z\vee (\bigvee_{1\le i\le m} \neg l^*_{i,j_i})$ where $j_i\in\{1,2,3\}$. As $\neg x_i\vee\neg x_i'\in\Pi$, we have $\neg x_i\vee\neg x_i'\not\models c$ for every $i~(1\le i\le n)$ due to $\Pi\not\models c$. It shows that $c$ mentions at most one atom in $\{x_i,x_i'\}$ for every $i$. Therefore $c$ corresponds to a satisfying assignment for $\gamma$.

  (ii) In the case that $\Sigma$ is unsatisfiable, i.e. $\Sigma\equiv\bot$,  $\forget(\Pi,V)\equiv\bot$ iff $\Pi\equiv\bot$.
   In this case the problem is tractable. Suppose $\Sigma$ is satisfiable.
   We have $\forget(\Pi,V)\models\Sigma$ iff $\forget(\Pi,V)\models c$ for every clause $c$ of $\Sigma$.
   % We may assume that $c$ contains no complement literals, i.e., $c$ is not equivalent to $\top$.
  In the case
  $\Var(c)\cap V\neq\emptyset$, we have $\forget(\Pi,V)\not\models c$. in the case $\Var(c)\cap V=\emptyset$,
  $\forget(\Pi,V)\models c$ iff $\Pi\models c$ iff $\Pi\cup\neg c$ is unsatisfiable, which is tractable even if
  $\Pi$ is a q-Horn theory \cite{Boros:AMAI:1990}.
  %It completes the proof. As a matter of fact,
%  $\forget(\Pi,V)\models \Sigma$ is computable in $O(|\Pi|\cdot|\Sigma|)$ of
% its input size, as $\Pi\models c$ is computable in linear time $O(|\Pi|+|c|)$.

  (iii)  Membership. If $\forget(\Pi,V)\not\models\forget(\Sigma,V)$ then there exists a prime implicate $c$ of
  $\Sigma$ such that $\Pi\not\models c$ and $\Var(c)\cap V=\emptyset$. Thus it is in \coNP.

  Hardness. It follows from (i) since $\forget(\Pi,V)\models\forget(\Sigma,V)$ iff $\Pi\models\forget(\Sigma,V)$.
  % Membership. In the case $\forget(\Pi,V)\not\models\Sigma$, there are two sets $M,M'$ of atoms such that
%  $M\models\Pi, M\setminus V=M'\setminus V$ and $M'\not\models\Sigma$. It is obvious that such guess and check are
%  polynomial.
%
%  Hardness. It is well-known that the validness of a DNF formula $\\Var$ is coNP-hard, even if every term $t$ of $\\Var$ has at most one negative literal as $\neg a\land\neg b$ in a term can be replaced with $\neg c$ by appending the three terms
%  . Note that
%  $\forget(\top,V)\equiv \top$, and $\top\models\Sigma$ iff $\neg\Sigma$ is unsatisfiable. As
%  $\neg\Sigma$ is a DNF formula in which each term has at most one negative literal.
  %(ii) It is trivially follows from the proof of (i).
\end{proof}

%In terms of the above proof, one can see that Theorem~\ref{thm:complexity:Horn} holds for $\Sigma$ being a CNF theory.
Accordingly, we have the following corollary.

\begin{corollary}\label{cor:complexity:forget:eq:Horn}
    Let $\Pi,\Sigma$ be two Horn (resp. ren-Horn and q-Horn) theories and $V\subseteq\cal A$.
    \begin{enumerate}[(i)]
      \item The problem of deciding if $\Pi\equiv\forget(\Sigma,V)$ is \coNP-complete.%, even if $\Pi$ and $\Sigma$ are acyclic.
      \item The problem of deciding if $\forget(\Pi,V)\equiv\forget(\Sigma,V)$ is \coNP-complete.%, even if $\Pi$ and $\Sigma$ are acyclic.
      \item The problem of deciding if $\forget(\Pi,V)\equiv\Pi$ is tractable.
    \end{enumerate}
\end{corollary}
\begin{proof}
  (i) As $\Pi\not\equiv\forget(\Sigma,V)$ iff $\Pi\not\models\forget(\Sigma,V)$ or $\forget(\Sigma,V)\not\models \Pi$,
  the latter is tractable by (ii) of Theorem \ref{thm:complexity:Horn} while
  the former is in \coNP. Hardness  follows from (i) of Theorem \ref{thm:complexity:Horn}. Thus the problem is \coNP-complete.

  (ii) Membership is easy. Hardness follows from (iii) of Theorem \ref{thm:complexity:Horn}.

  (iii) It follows from the facts that $\forget(\Pi,V)\equiv\Pi$ iff $\forget(\Pi,V)\models\Pi$, and
  (ii) of Theorem \ref{thm:complexity:Horn}.
\end{proof}

The item (iii) in the above corollary shows that the problem of deciding whether $\Pi$ is relevant to $V$ is tractable if
$\Pi$ is a q-Horn theory. Thus it generalizes Theorem~51 of \cite{Lakemeyer:AIJ:1997} for Horn theories.

\subsubsection{Krom theories}
Note that, for every Krom theory $\Sigma$ and $V\subseteq\cal A$. It is evident that
\[\Forget(\Sigma,V)\equiv\{l_1\lor l_2|\Var(\{l_1,l_2\})\subseteq\Var(\Sigma)\setminus V\ \textrm{and}\ \Sigma\models l_1\lor l_2\}.\]
It implies that $\Forget(\Sigma,V)$ can be computed in polynomial time in the size of $\Sigma$ and $V$ since $\Sigma\models l_1\lor l_2$
is tractable \cite{Krom:JSL:1970} and there are at most $O(|\Var(\Sigma)\setminus V|^2)$ number of such clauses. The following corollary follows.

\begin{corollary}\label{cor:complexity:forget:eq:Krom}
  Let $\Pi,\Sigma$ be two Krom theories and $V\subseteq\cal A$. All of the following problems are tractable:
  \begin{enumerate}[(i)]
    \item deciding if $\Pi\models\forget(\Sigma,V)$,
    \item deciding if $\forget(\Pi,V)\models\Sigma$,
    \item deciding if $\forget(\Pi,V)\models\Forget(\Sigma,V)$,
    \item deciding if $\Pi\equiv \forget(\Sigma,V)$,
    \item deciding if $\forget(\Pi,V)\equiv\forget(\Sigma,V)$,
    \item deciding if $\forget(\Pi,V)\equiv\Pi$.
  \end{enumerate}
\end{corollary}

%\begin{corollary}
%  Let $T$ be a theory, $V$ a set of atoms, $q$ an atom in $T$ but not in $V$, and $\var$ a formula.
%  \begin{enumerate}[(i)]
%    \item The problem of deciding if $\var$ is the strongest necessary condition of $q$ on $V$ is $\Pi_2^P$-complete even if both $T$ and $\var$ are CNF formulae; it is \coNP-complete when both $T[q/\bot]$ and $\var$ are  Horn formulae.
%    \item The problem of deciding if $\var$ is the weakest sufficient condition of $q$ on $V$ is $\Pi_2^P$-complete even if both $T$ and $\var$ are CNF formulae; it is \coNP-complete when both $T[q/\bot]$ and $\var$ are Horn formulae.
%  \end{enumerate}
%\end{corollary}

\section{Related Work}\label{sec:application}
In the section we consider the applications of forgetting, including
uniform interpolation \cite{DAgostino:synthese:2008}, strongest necessary and weakest sufficient conditions \cite{DBLP:Lin:AIJ:2001}, and
strongest and weakest definitions \cite{Lang:AIJ:2008}.%, in particular for the case of q-Horn theories.
\subsection{Uniform interpolation}

Let $\alpha,\beta$ be two formulas. If $\alpha\models\beta$, an {\em interpolant for} $(\alpha,\beta)$ is a formula $\gamma$ s.t
\begin{equation}\label{interpolant}
  \alpha\models\gamma\quad \mbox{and}\quad \gamma\models\beta
\end{equation}
where $\Var(\gamma)\subseteq \Var(\alpha)\cap \Var(\beta)$.

A logic $\cal L$ with inference
$\models_{\cal L}$ is said to have the {\em interpolantion property} if an interpolant exists for every pair of
formulas $(\alpha,\beta)$ such that $\alpha\models_{\cal L}\beta$.
A logic $\cal L$ has {\em uniform interpolation property} iff for any formula $\alpha$ and $V$ a set of atoms,
there exists a formula $\gamma$ such that $\Var(\gamma)\subseteq \Var(\alpha)\setminus V$, and for any formula $\beta$ with
$\Var(\beta)\cap V=\emptyset$,
\begin{equation}
  \alpha\models_{\cal L}\beta\quad \mbox{iff}\quad \gamma\models_{\cal L}\beta.
\end{equation}
It is easy to see that uniform interpolation is a strengthening of interpolation. A well-known result
is that propositional logic has uniform interpolation property, while first-order logic does not \cite{DAgostino:synthese:2008}.

\begin{proposition}\label{prop:double:horn:forget}
  If $\Sigma$ is a double Horn theory and $V\subseteq \cal A$ then
  $\Forget(\Sigma,V)$ is a double Horn theory.
\end{proposition}
\begin{proof}
  Firstly $\Forget(\Sigma,V)$ is Horn expressible by Corollary~\ref{cor:Horn:variant:forget}. We show
  that $\neg\Forget(\Sigma,V)$ is Horn expressible by contradiction in the following.
  Suppose that there exist two interpretations $X,Y$ such that
  \begin{align*}
    X\not\models\Forget(\Sigma,V),\ Y\not\models\Forget(\Sigma,V),\ X\cap Y\models\Forget(\Sigma,V).
  \end{align*}
  Note that $\Forget(\Sigma,V)$ is irrelevant to $V$. Thus $I\models\Forget(\Sigma,V)$ if and only if
  $I\setminus V\models\Forget(\Sigma,V)$. For this reason, we assume $X\cap V=\emptyset$ and $Y\cap V=\emptyset$.
  The following three conditions hold:
  \begin{itemize}
    \item[(a)] $X'\not\models\Sigma$ for any $X\subseteq X'\subseteq X\cup V$.
    \item[(b)] $Y'\not\models\Sigma$ for any $Y\subseteq Y'\subseteq Y\cup V$.
    \item[(c)] There exists $Z\models\Sigma$ for some $ X\cap Y\subseteq Z\subseteq X\cap Y\cup V$.
  \end{itemize}
  The conditions (a) and (b) imply $X'\cap Y'\not\models\Sigma$ since $\Sigma$ is a double Horn formula.
  It is evident that $X\cap Y\subseteq X'\cap Y'\subseteq (X\cup V)\cap (Y\cup V)=X\cap Y\cup V$.
  This contradicts with condition (c).
\end{proof}

Together with Corollary~\ref{cor:Horn:variant:unfold}, the proposition above implies:
\begin{corollary}\label{cor:uniform:interpolation}
  The Horn, Krom, double Horn, ren-Horn and q-Horn fragments of propositional
  logic have uniform interpolation property.
\end{corollary}
%
%The following example shows that the forgetting result of a q-Horn theory may be not q-Horn.
%\begin{example}
%  Let $\Pi=\{p_1\lor p_2\lor \neg q_1, q_1\lor q_2\}$ and $V=\{q_1\}$. Note that
%  $\PI(\Pi)=\Pi\cup\{p_1\lor p_2\lor q_2\}$ and then $\Forget(\Pi,V)=\{p_1\lor p_2\lor q_2\}$ by
%  Theorem~\ref{thm:forget:1}. It is not difficult to verify that $\Forget(\Pi,V)$ is not q-Horn.
%  However $\Pi$ is q-Horn since it has a QH-partition with $Q=\{p_1,p_2,q_1,q_2\}$ and $H=\emptyset$.
%\end{example}
%
%\subsection{Independence}
%
%A formula $\var$ is \emph{independent} of from a set $X$ of atoms if and only if there is a formula $\psi$ such that
%$X\cap \Var(\psi)=\emptyset$ and $\var\equiv\psi$ \cite{Lang:JAIR:2003}. It is well-known that
%$\var$ is independent of a set $X$ of atoms iff $\var\equiv\Forget(\var,X)$, cf. Corollary~7 of \cite{Lang:JAIR:2003}.
%Thus we have, by Corollaries~\ref{cor:complexity:forget:eq:Horn} and \ref{cor:complexity:forget:eq:Krom},
%\begin{corollary}
%  Let $\var$ be a formula and $X\subseteq\cal A$. The problem of deciding if $\var$ is independent of $X$
%  is tractable if $\var$ is a Krom or q-Horn formula.
%\end{corollary}

\subsection{Strongest necessary and weakest sufficient conditions}
%Forgetting is closed related to the notions of strongest necessary conditions and weakest sufficient conditions \cite{DBLP:Lin:AIJ:2001}.
Let $T$ be a theory, $V\subseteq\Var(T)$ and $q\in\Var(T)\setminus V$. A formula $\varphi$ of $V$ is a {\em necessary condition} of $q$ on $V$ under $T$ if $T\models q\supset\varphi$.
  It is a {\em strongest necessary condition (SNC)} if it is a necessary condition and for any other necessary condition $\varphi'$, $T\models\varphi\supset\varphi'$.  A formula $\psi$ of $V$ is a {\em sufficient condition} of $q$ on $V$ under $T$ if $T\models\psi\supset q$.
  It is a {\em weakest sufficient condition (WSC)} if it is a sufficient condition and, for any other sufficient condition
  $\psi'$, $T\models\psi'\supset\psi$ \cite{DBLP:Lin:AIJ:2001}.
  %
%\begin{definition}[\cite{DBLP:Lin:AIJ:2001}]\label{def:NC:SC}
%  Let $T$ be a theory, $V\subseteq\Var(T)$ and $q\in\Var(T)\setminus V$.
%  \begin{itemize}
%  \item
%  A formula $\varphi$ of $V$ is a {\em necessary condition} of $q$ on $V$ under $T$ if $T\models q\supset\varphi$.
%  It is a {\em strongest necessary condition (SNC)} if it is a necessary condition and for any other necessary condition $\varphi'$, $T\models\varphi\supset\varphi'$.
%  \item
%  A formula $\psi$ of $V$ is a {\em sufficient condition} of $q$ on $V$ under $T$ if $T\models\psi\supset q$.
%  It is a {\em weakest sufficient condition (WSC)} if it is a sufficient condition and, for any other sufficient condition
%  $\psi'$, $T\models\psi'\supset\psi$.
%  \end{itemize}
%\end{definition}

\begin{theorem}[Theorem~2 of \cite{DBLP:Lin:AIJ:2001}]\label{thm:forget:SNC:WSC}
   Let $T$ be a theory, $V\subseteq\Var(T)$, $q\in\Var(T)\setminus V$, and
  $V'=\Var(T)\setminus (V\cup\{q\})$.
  \begin{itemize}
    \item The strongest necessary condition of $q$ on $V$ under $T$ is $\Forget(T[q/\top],V')$.
    \item The weakest sufficient condition of $q$ on $V$ under $T$ is $\neg\Forget(T[q/\bot],V')$.
  \end{itemize}
\end{theorem}

Note that $T[q/\top]$ is a Horn (resp. Krom, ren-Horn and q-Horn) theory if $T$ is a Horn
(resp. Krom, ren-Horn and q-Horn) theory. In terms of Corollary~\ref{cor:uniform:interpolation},
the SNC of $q$ under $T$ is Horn (resp. Krom, ren-Horn and q-Horn) expressible if $T$ is a Horn
(resp. Krom, ren-Horn and q-Horn) theory.

The following example shows that the weakest sufficient condition on Horn (resp. Krom)
formulas may be not Horn (resp. Krom) expressible.
\begin{example}
Let's consider the following two theories.

%\begin{itemize}
%  \item
 (1) Let $\Sigma=(\neg p\lor\neg r)\land(\neg q\lor r)\land(\neg s\lor r)\land\neg t$, which is a Horn formula. We have that
  $\Forget(\Sigma[t/\bot],r)\equiv (\neg p\lor \neg q)\land(\neg p\lor \neg s)$. Thus
  $\neg\Forget(\Sigma,r)\equiv p\land(q\lor s)$, which is evidently not Horn expressible. That is
  the weakest sufficient condition of $t$ on $\{p,q,s\}$ under $\Sigma$ is not Horn expressible.

  %\item
  (2) Let $\Pi= (p_1\lor p_2)\land (\neg p_1\lor p_3)\land (\neg p_2\lor \neg p_3)\land \neg q$, which is a Krom formula.
  Note that
  $\Forget(\Pi[q/\bot],\emptyset)\equiv (p_1\lor p_2)\land (\neg p_1\lor p_3)\land (\neg p_2\lor \neg p_3)$. Thus
  $\neg\Forget(\Pi[q/\bot],\emptyset)\equiv (\neg p_1\lor p_2\lor \neg p_3)\land (p_1\lor\neg p_2\lor p_3)\land (\neg p_2\lor\neg p_3)$.
  It is not a Krom formula. Actually, the clause $\neg p_1\lor p_2\lor \neg p_3$ is a prime implicate of $\neg\Forget(\Pi[q/\bot],\emptyset)$.
%\end{itemize}
\Eed
\end{example}

\begin{theorem}\label{thm:complexity:SNC:WSC:Horn}
  Let $T,\var$ be two formulas, $V\subseteq\Var(T)$, $q\in\Var(T)\setminus V$.
  \begin{enumerate}[(i)]
    \item Deciding if $\var$ is a necessary (sufficient) condition of $q$ under $T$ is \coNP-complete.
    \item Deciding if $\var$ is a necessary (sufficient) condition of $q$ under $T$ is tractable if
    $T$ and $\var$ are Horn (resp. ren-Horn and q-Horn) formulas.
    \item Deciding if $\var$ is a strongest necessary (weakest sufficient) condition of $q$ under $T$ is \PIP{2}-complete.
    \item Deciding if $\var$ is a strongest necessary (weakest sufficient) condition of $q$ under $T$ is \coNP-complete if
    $T$ and $\var$ are Horn (resp. ren-Horn and q-Horn) formulas.
  \end{enumerate}
\end{theorem}
\begin{proof}
  (i) $T\models q\supset\var$ iff $T\land q\land\neg\var$ is unsatisfiable. This is in \coNP\ and \coNP-hard, i.e.
  deciding if $\var$ is a necessary condition of $q$ under $T$ is \coNP-complete. The case of sufficient condition is similar.

  (ii) $T\models q\supset\var$ iff $T\land q\land \neg c$ is unsatisfiable for every clause $c$ of $\var$, which is tractable
  even $T$ and $\var$ are q-Horn formulas.
  Thus deciding if $\var$ is a necessary condition of $q$ under $T$ is tractable. Similarly
  $T\models\var\supset q$ iff $T\land\var\land\neg q$ is unsatisfiable even if $T$ and $\var$ are
  q-Horn formulas.

  (iii) In terms of Theorem~\ref{thm:forget:SNC:WSC}, $\var$ is a strongest necessary condition
  of $q$ under $T$ iff $\var\equiv\Forget(T[q/\top],V')$ where $V'=\Var(T)\setminus (V\cup\{q\})$. It is in \PIP{2} and
  \PIP{2}-hard by (i) of Corollary~\ref{cor:forget:eq}.

  (iv) Recall that $\var$ is a strong necessary condition of $q$ under $T$ if and only if $\var\equiv\Forget(T[q/\top],V')$ by (i)
  of Theorem~\ref{thm:forget:SNC:WSC} where $V'=\Var(T)\setminus(V\cup\{q\})$. Thus it is in \coNP\ when $\var$ and $T$ are q-Horn formulas
  and is \coNP-hard   when $\var$ and $T$ are Horn formulas by (i) of Corollary~\ref{cor:complexity:forget:eq:Horn}.
\end{proof}

\begin{proposition}\label{prop:complexity:SNC:WSC:Krom}
  Let $T$ and $\var$ be two Krom formulas, $V\subseteq\Var(T)$, $q\in\Var(T)\setminus V$.
  \begin{enumerate}[(i)]
    \item  Deciding if $\var$ is a strongest necessary condition of $q$ under $T$ is tractable.
    %\item Deciding if $\var$ is a  sufficient condition of $q$ under $T$ is tractable.
    \item Deciding if $\var$ is a  weakest sufficient condition of $q$ under $T$ is tractable.
  \end{enumerate}
\end{proposition}
\begin{proof}
  Firstly, according to Theorem~\ref{thm:forget:1}, one can compute $\Forget(T[q/\top],V')$ in polynomial time in the size of $T$ and $V$
  where $V'=\Var(T)\setminus (V\cup\{q\})$.  It is evident that $\Sigma=\unfold(\Forget[q/\top],V')$
  and $\Sigma'=\Forget(T[q/\bot],V')$ are Krom theories.

  (i) It follows from the facts that checking equivalence for Krom theories is tractable and
  $\var$ is a strongest condition of $q$ under $T$ iff $\var\equiv\Sigma$ by (i) of Theorem~\ref{thm:forget:SNC:WSC}.

  (ii) $\var$ is a weakest sufficient condition of $q$ under $T$\\
  iff $\var\equiv\neg\Sigma'$\\
  iff $\var\models\neg\Sigma'$ and $\neg\Sigma'\models\var$\\
  iff $\var\land\Sigma'$ is unsatisfiable and $\neg \Sigma'\models l_1\lor l_2$ for every conjunct $l_1\lor\l_2$ of $\var$.

  It is evident that checking satisfiability of $\var\land\Sigma'$ is tractable since $\var\land\Sigma'$ is a Krom formula.
  Note further that
  $\neg \Sigma'\models l_1\lor l_2$ \\
  iff $\neg\Sigma'\land\neg l_1\land\neg l_2$ is unsatisfiable\\
  iff $\Sigma''=\neg(\Sigma'[\neg l_1/\top][\neg l_2/\top])$ is unsatisfiable\\
  iff $s_1\land s_2$ is unsatisfiable for every disjunct $s_1\land s_2$ of $\Sigma''$, which is a 2-DNF formula.
\end{proof}

\subsection{Strongest and weakest definitions}
Definability is acknowledged as an important logical concept when reasoning about
knowledge represented in propositional logic. Informally speaking,
an atom $p$ can be ``defined'' in a given formula $\Sigma$ in terms of a set $X$ of atoms whenever the knowledge of
the truth values of $X$ enables concluding about the truth value of $p$, under the condition of $\Sigma$ \cite{Lang:AIJ:2008}.

\begin{definition}[\cite{Lang:AIJ:2008}]\label{def:definability}
 Let $\Sigma$ be a  formula, $p\in\cal A$, $X\subseteq\cal A$ and $Y\subseteq\cal A$.
    \begin{itemize}
    \item $\Sigma$ \emph{defines} $p$ in terms
    of $X$, denoted by $X\sqsubseteq_\Sigma p$, iff there exists a formula $\Psi$ over $X$ such that $\Sigma\models\Psi\lrto p$.
    \item $\Sigma$ \emph{defines} $Y$ in terms of $X$, denoted by $X\sqsubseteq_\Sigma Y$, iff
    there exists a formula $\Psi$ over $X$ such that $\Sigma\models\Psi\lrto p$ for every $p\in Y$.
    \end{itemize}
\end{definition}

It is known that if both $\var$ and $\psi$ (over a same signature $X$) are definitions of $p$ in $\Sigma$ then $\Sigma\models\var\lrto\psi$,
and additionally both $\var\land\psi$ and $\var\lor\psi$ are definitions of $p$ in $\Sigma$.
In this situation, the \emph{strongest}
(resp. \emph{weakest}) definition of $p$ in $\Sigma$ exist, they are denoted by
$\textrm{Def}^{X,l}_\Sigma(p)$ and $\textrm{Def}^{X,u}_\Sigma(p)$ respectively. In terms of Corollary~9 of \cite{Lang:AIJ:2010}
and Theorem~10 of \cite{Lang:AIJ:2008},
if $\Sigma$ defines $p$ in terms of $X$ then
$\textrm{Def}^{X,l}_\Sigma(p)$ (resp. $\textrm{Def}^{X,u}_\Sigma(p)$) is equivalent to
the strongest necessary (resp. weakest sufficient) condition of $p$ under $\Sigma$. Thus according
to Theorem~\ref{thm:complexity:SNC:WSC:Horn} and Proposition~\ref{prop:complexity:SNC:WSC:Krom} we have the following:
\begin{corollary}
  Let $\Sigma,\var$ be two formulas, $X\subseteq\cal A$, $p\in\cal A$ and $\Var(\var)\subseteq X$.% such that  $\Sigma$ defines $p$ in terms of $X$.
  \begin{enumerate}[(i)]
    \item The problem of deciding if $\var$ is a strongest (resp. weakest) definition of $p$ (in terms of $X$) in $\Sigma$
    is \PIP{2}-complete.
    \item The problem of deciding if $\var$ is a strongest (resp. weakest) definition of $p$ (in terms of $X$) in $\Sigma$
    is \coNP-complete if both $\Sigma$ and $\var$ are Horn (resp. ren-Horn and q-Horn) formulas.
    \item deciding if $\var$ is a strongest (resp. weakest) definition of $p$ (in terms of $X$) in $\Sigma$
    is tractable if both $\Sigma$ and $\var$ are Krom formulas.
  \end{enumerate}
\end{corollary}
%
%\subsection{Difference}
%
%The \emph{difference} of two formulas $\var$ and $\psi$ is logically equivalent to the formula $\var\land\neg\psi$. It is tractable
%to check whether $\var\land\neg\psi$ is Horn expressible if both $\var$ and $\psi$ are Horn formulas \cite{Eiter:JCSS:2000}.
%
%\begin{definition}\label{def:difference}
%  Let $\var,\psi$ be two formulas and $X\subseteq\cal A$. The \emph{difference between $\var$ and $\psi$ on $X$}, written $\Diff(\var,\psi,X)$,
%  is the theory
%  \[\{\xi|\Var(\xi)\subseteq X,\ \var\models\xi\ \&\ \psi\not\models\xi\}.\]
%\end{definition}

\section{Concluding Remarks}\label{sec:related-work}
As mentioned in the introduction, forgetting is closely connected with many other logical concepts.
Quite late, the notion of relevance was quantitatively investigated  \cite{Liang:CCAI:2013},
and the notion of independence was applied to belief change \cite{Marquis:AIJ:2014}, which
is a long-standing and vive topic in AI \cite{AGM:JSL85}. The main concerned
Horn, Krom and other fragments of propositional logic are also ubiquitous in AI
\cite{Bubeck:TAST:2005,Liberatore:AIJ:2008,Maonian:AAAI:2011,Delgrande:JAIR:2013,DBLP:journals/ai/DelgrandeP15}.

In the paper we have firstly presented a resolution-based algorithm for computing forgetting results of
CNF fragments of propositional logic. Though the algorithm is generally expensive even for Horn
fragment as it is theoretically intractable, it opens a
heuristic potentiality, e.g. choosing different orders of atoms to forget, and choosing different orders
of resolvable clauses to do resolution. To investigate the effectiveness of the algorithm, heuristics and extensive experiments
are worthy of studying.

What's more, when concerning the dynamics of knowledge base, we  considered various
reasoning problems about forgetting in the fragments of propositional logic whose
satisfiability are tractable.
In particular, we concentrated on Horn, renamable Horn, q-Horn and Krom theories.
The considered reasoning problems
include \VarEQ, \VarIND, \VarWeak, \VarStrong, \VarMatch\
and \VarEnt. Although some of the problems  have been partially solved, e.g., \VarEQ\ and \VarIND\ for
propositional logic are proved in \cite{Lang:JAIR:2003}, this is the first comprehensive study on these problems for
CNF, Horn, ren-Horn, q-Horn, Krom and DNF fragments, to our knowledge. It motivates us to
consider these reasoning problems for forgetting in
non-classical logical systems, such as model logic S5 in particular.

%We also showed that some of the results of the paper are applicable to other logical notions,
%including uniform interpolation, strongest necessary and weakest sufficient conditions.
It deserves our further effort to investigate the knowledge simplification or compilation \cite{Bienvenu:AAAI:2010}
in other logical formalisms, logic programming under stable model semantics, particularly.
\\
\\

\textbf{Acknowledgement} This work was  supported by the National Natural Science Foundation of China under
grants 60963009,61370161 and Stadholder Foundation of Guizhou Province under grant (2012)62.
%\section{Concluding Remarks}\label{sec:conclusion}
%
%In the paper we have considered the forgetting in fragments of propositional logic, including
%Horn theories and its variants, and Krom theories. A general algorithm based on resolution to
%computing forgetting result has been presented. As to the reasoning problems relating to forgetting,
%we considered the inference problems in the fragments of Horn theories and its variants.

\bibliographystyle{unsrt}

%\bibliography{E:/docs/DL}

\end{document}